\documentclass{article}

\usepackage{environ}
\newif\ifconferencemode
\NewEnviron{movable}[2]{
  \ifconferencemode
    #2 
    \expandafter\xdef\csname stored#1\endcsname{\unexpanded\expandafter{\BODY}}
  \else
    \BODY
  \fi
}
\newcommand{\showstored}[1]{\csname stored#1\endcsname}

\conferencemodefalse

\ifconferencemode
    \newcommand{\conf}[1]{#1}
    \newcommand{\Arxiv}[1]{}
    \newcommand{\arxvFtnt}[1]{}
    \newcommand{\MCfig}{delta_comparison_1_panel.png}
\else
    \newcommand{\conf}[1]{}
    \newcommand{\Arxiv}[1]{#1}
    \newcommand{\arxvFtnt}[1]{\footnote{#1}}
    \newcommand{\MCfig}{delta_comparison_2_panels.png}
\fi

\ifconferencemode
    \usepackage[accepted]{./style/ICML_2026}
\else
    \usepackage{natbib}
    \usepackage[margin=1in]{geometry}
    \usepackage{bbm}
    \usepackage{algorithm, algorithmic}
    \usepackage[textsize=tiny]{todonotes}
\fi

\usepackage{subcaption}
\usepackage{wrapfig}
\usepackage{microtype}
\usepackage{graphicx}
\usepackage{booktabs}
\usepackage{hyperref}
\usepackage{amsmath}
\usepackage{amssymb}
\usepackage{mathtools}
\usepackage{amsthm}
\usepackage{bbm}
\usepackage[capitalize,noabbrev]{cleveref}

\theoremstyle{plain}
\newtheorem{theorem}{Theorem}[section]

\newtheorem{lemma}[theorem]{Lemma}

\newtheorem{claim}[theorem]{Claim}
\theoremstyle{definition}
\newtheorem{definition}[theorem]{Definition}

\theoremstyle{remark}
\newtheorem{remark}[theorem]{Remark}
\newcommand{\reals}{\mathbb{R}}
\newcommand{\naturals}{\mathbb{N}}

\newcommand{\lvec}[1]{\reflectbox{$\vec{\reflectbox{$#1$}}$}}
\newcommand{\para}[1]{\paragraph{#1}}
\newcommand{\eps}{\varepsilon}

\newcommand{\out}{y}

\newcommand{\outDom}{\mathcal{Y}}
\newcommand{\alg}{M}
\newcommand{\rand}{R}
\newcommand{\domain}{\mathcal{X}}
\newcommand{\dataset}{\boldsymbol{s}}
\newcommand{\view}{\boldsymbol{v}}

\newcommand{\prob}[2]{\underset{#1}{\mathbb{P}} \left(#2 \right)}
\newcommand{\expect}[2]{\underset{#1}{\mathbb{E}} \left[#2 \right]}
\newcommand{\loss}{\ell}
\newcommand{\lossFunc}[3]{\loss\left(#1; #2, #3 \right)}

\newcommand{\upBnd}{U}
\newcommand{\lowBnd}{V}
\newcommand{\lossRV}{L}
\newcommand{\dual}[1]{\mathcal{D}\left(#1\right)}
\newcommand{\lossRVfunc}[2]{\lossRV_{#1, #2}}

\newcommand{\HockeyStick}[1]{\boldsymbol{H}_{#1}}
\newcommand{\HockeyStickFunc}[2]{\HockeyStick{#1}\left(#2 \right)}
\newcommand{\HockeyStickDiv}[3]{\HockeyStick{#1}\left(#2 ~\left\Vert~ #3 \right. \right)}

\newcommand{\alloc}[2]{\mathcal{A}_{#1}\left(#2\right)}
\newcommand{\allocFunc}[3]{\mathcal{A}_{#1}\left(#2; #3\right)}

\newcommand{\add}[1]{\lvec{#1}}
\newcommand{\rem}[1]{\vec{#1}}

\newcommand{\privProf}[1]{\delta_{#1}}
\newcommand{\privProfAdd}[1]{\add{\delta}_{#1}}
\newcommand{\privProfRem}[1]{\rem{\delta}_{#1}}

\ifconferencemode
    \icmltitlerunning{Efficient Privacy Loss Accounting for Subsampling and Random Allocation}
\fi

\begin{document}

\ifconferencemode
\else
    \title{Efficient Privacy Loss Accounting for Subsampling and Random Allocation}
    \author{
        Vitaly Feldman\\
        Apple \\
        \and
        Moshe Shenfeld\footnote{Work partially done while author was an intern at Apple.}\\
        The Hebrew University of Jerusalem
    }
    \maketitle
\fi

\begin{abstract}
    We consider the privacy amplification properties of a sampling scheme in which a user's data is used in $k$ steps chosen randomly and uniformly from a sequence (or set) of $t$ steps. This sampling scheme has been recently applied in the context of differentially private optimization \citep{CGHLKKMSZ24, CCGHST25} and communication-efficient high-dimensional private aggregation \citep{AFKRT26}, where it was shown to have utility advantages over the standard Poisson sampling. Theoretical analyses of this sampling scheme \citep{FS25, DCO25} lead to bounds that are close to those of Poisson sampling, yet still have two significant shortcomings. First, in many practical settings, the resulting privacy parameters are not tight due to the approximation steps in the analysis. Second, the computed parameters are either the hockey stick or R\'{e}nyi divergence, both of which introduce overheads when used in privacy loss accounting.
    In this work, we demonstrate that the privacy loss distribution (PLD) of random allocation applied to any differentially private algorithm can be computed efficiently. When applied to the Gaussian mechanism, our results demonstrate that the privacy-utility trade-off for random allocation is at least as good as that of Poisson subsampling. In particular, random allocation is better suited for training via DP-SGD.  To support these computations, our work develops new tools for general privacy loss accounting based on a notion of PLD realization. This notion allows us to extend accurate privacy loss accounting to subsampling which previously required manual noise-mechanism-specific analysis.
\end{abstract}

\section{Introduction}
Privacy amplification by data sampling is one of the central techniques in the analysis of differentially private (DP) algorithms. In this technique, a differentially private algorithm (or a sequence of DP algorithms) is executed on a randomly chosen set of data elements without revealing which of the elements were used.
As first demonstrated by \citet{KLNRS11} this additional randomness can significantly improve the privacy guarantees of the resulting algorithm, that is, privacy amplification. 

Privacy amplification by sampling has found numerous applications, most notably in the analysis of the differentially private stochastic gradient descent (DP-SGD) algorithm \citep{BST14} for training neural networks with differential privacy. 
In DP-SGD the gradients are computed on randomly chosen batches of data points and then privatized through Gaussian noise addition. Privacy analysis of this algorithm is based on the so-called Poisson sampling: elements in each batch and across batches are chosen randomly and independently of each other. The absence of dependence implies that the algorithm can be analyzed relatively easily as an independent composition of single step amplification results. This simplicity is also the key to accurate numerical accounting of the privacy parameters of DP-SGD that are crucial for all existing practical applications of DP-SGD. 

The downside of the simplicity of Poisson sampling is that independently resampling every batch is less efficient and harder to implement within the standard ML pipelines. As a result, in practice typically some form of data shuffling is used to define the batches in DP-SGD even though the privacy analysis relies on Poisson sampling (e.g.,~\citep{MHSBCCGKKL25}).
Data shuffling in which the elements are randomly permuted before being assigned to steps of the algorithm is also known to lead to privacy amplification. However, the analysis of this sampling scheme is more involved and nearly tight numerical results are known only for relatively simple pure DP ($\delta =0$) algorithms \citep{EFMRTT19,FMT21,FMT23, GDDKS21, GDDSK21}.  In particular,  for the case of Gaussian noise addition there is no practically useful method of computing the privacy parameters of DP-SGD with shuffling. 

The discrepancy between the implementations of DP-SGD and their analysis has been explored in several recent works demonstrating that shuffling can be less private than Poisson subsampling \citep{CGKKMSZ24a, CGKKMSZ24b, ABDCH26}. Motivated by these findings, \citet{CGHLKKMSZ24} study training of neural networks via DP-SGD with batches sampled via {\em balls-and-bins sampling}. In this sampling scheme, each data element is assigned randomly and independently (of other elements) to exactly one out of $t$ possible batches. Their main results show that from the point of view of utility (namely, accuracy of the final model) such sampling is essentially identical to shuffling and is noticeably better than Poisson sampling. Concurrently, \citet{CCGHST25} considered the same sampling scheme for the matrix mechanism in the context of DP-FTRL. The privacy analysis in these two works reduces the problem to analyzing the divergence of a specific pair of distributions on $\reals^t$. They then used Monte Carlo simulations to estimate the privacy parameters of this pair. This estimation method was improved in a follow-up work by \citet{DG26}, by framing it as a dynamic programming problem. These simulations provide strong evidence that privacy guarantees of balls-and-bins sampling for Gaussian noise are similar to those of the Poisson sampling with rate $1/t$. While very encouraging, such simulations do not establish formal guarantees. In addition, achieving high-confidence estimates for small $\delta$ and supporting composition appear to be computationally impractical. 

Another important application of privacy amplification is for reducing communication in private federated learning \citep{CSOK24,AFKRT26, DCO25}. In this application, each user subsamples the coordinates of the vector it holds (typically representing a model update) and then communicates the selected coordinates. Secure aggregation protocols are used to ensure that the server does not learn which coordinates were sampled by which user, thereby achieving privacy amplification. In this setting, it is also typically necessary to limit the maximum number of coordinates a user sends due to computational or communication constraints on the protocol. Poisson subsampling results in a random (binomial) number of coordinates to communicate and thus does not allow to fully exploit the available limit. Thus in \citep{AFKRT26}, a natural alternative sampling scheme is considered in which each user contributes a random $k$ out of the total $t$ times (but with users still doing this independently). For $k=1$ this sampling scheme corresponds to the balls-and-bins sampling \citep{CGHLKKMSZ24}.

Motivated by the applications above, \citet{FS25} propose and analyze a general sampling scheme where each element participates in exactly $k$ randomly chosen steps out of the total $t$, independently of other elements, referred to as $k$-out-of-$t$ {\em random allocation}. They show a reduction of the general $k$ scheme to $k=1$ and describe several ways to analyze the $1$-out-of-$t$ sampling scheme for general differentially private algorithms. \citet{DCO25} independently derived an additional analysis of the privacy of $k$-out-of-$t$ random allocation for Gaussian noise addition via R\'{e}nyi divergence. Recently, \citet{SK26} proposed an efficient version of the R\'{e}nyi divergence computation using dynamic programming.

The analyses in \citep{FS25,DCO25} and the numerical methods they entail demonstrate that in most practical settings the privacy amplification achieved by random allocation is comparable to that of Poisson sampling with the best results being typically within $20\%$ increase in $\eps$. While reasonably close, these bounds are worse than the bounds estimated via Monte Carlo simulations \citep{CGHLKKMSZ24,CCGHST25, DG26} and bounds that can be computed exactly in some special cases \citep{FS25, SK26}. Further, these analyses bound either the $(\eps,\delta)$ parameters \citep{FS25} or the R\'{e}nyi DP parameters \citep{FS25,DCO25, SK26} of the resulting algorithm. Both of these bounds have important limitations when used with additional processing steps. For example, the algorithm used in \citep{AFKRT26} relies on random allocation to reduce communication for each user but on top of it uses DP-SGD to sample batches of users using Poisson sampling and composition (for batches and epochs). In such an application, using an $(\eps,\delta)$-bound for random allocation would require performing composition for general $(\eps,\delta)$ algorithms which is known to be suboptimal\footnote{Suboptimality of $(\eps,\delta)$ composition is the main reason for the introduction of moment accountant technique and the development of numerical composition tools.}.
On the other hand, the general subsampling bounds based on R\'{e}nyi DP are typically loose. Further, conversion from R\'{e}nyi DP to final $(\eps, \delta)$ guarantees also typically introduces overheads. 

\subsection{Our Contribution}
We demonstrate how to overcome both shortcomings of the existing numerical methods for computing the privacy parameters of random allocation. Specifically, we show a method that, given a privacy loss distribution (PLD) for a pair of distributions that dominates the privacy loss of each step of some sequence of $t$ differentially private algorithms, computes an upper bound on the PLD of the $1$-out-of-$t$ random allocation applied to that sequence of algorithms. 

Our algorithm is efficient in that, for the Gaussian mechanism, its running time is $O(\log^3 (t) \cdot \log(t/\beta)/(\sigma^{2}\alpha^2))$ (Theorem \ref{thm:num_acc_RA} gives the general bound in terms of the interquantile range), where $\alpha$ is the approximation parameter of the loss (roughly corresponding to the error in $\eps$) and $\beta$ is an additional probability of unbounded loss (translating to an increase in $\delta$), as demonstrated in Figure \ref{fig:runtime_experiment}. 
Combining this with the reduction from the general $k$ case to $k=1$ (Lemma \ref{lem:multAlloc}), we also obtain an algorithm for computing the PLD of the $k$-out-of-$t$ random allocation.

PLD is now the standard representation of privacy loss used in privacy accounting libraries (e.g.,~\citep{Google20,Microsoft21,Meta21}) when computing composition. Its primary benefit is that it can be efficiently and losslessly composed as well as converted to other notions of DP such as $(\eps,\delta)$-DP and R\'{e}nyi DP. We demonstrate that PLD-based computations can be used for more general privacy accounting (and not just composition). Specifically, we show how to implement Poisson subsampling directly on PLDs.\footnote{In contrast, existing privacy accounting of subsampling relies on analytic expressions of the PLD of subsampled Gaussian/Laplace mechanism.} Our implementation crucially relies on ensuring that the approximate representation of a PLD we aim to compute is itself a {\em PLD realization}; that is, it corresponds to a PLD of some pair of distributions that dominates the algorithm. We also point out that our algorithm for random allocation naturally supports PLD-based accounting. Together, these algorithms enable accurate privacy accounting for more complex algorithms, such as the nested sampling used in DP-SGD with the PREAMBLE scheme  \citep{AFKRT26}. An implementation of these algorithms can be found at \citep{PLD_acc}. See Section~\ref{sec:PLDreal} for a detailed overview of this contribution.

\para{Technical overview:}
We now briefly outline our approach. As in the prior work, the starting point of our result is a relatively simple fact that a dominating pair of distributions\footnote{Informally, a pair of distributions is dominating for $M$ if it realizes (an upper bound on) all the worst case privacy parameters of $M$ (see Defn.~\ref{def:domPair}).}  for a  $1$-out-of-$t$ random allocation applied to a $t$-step algorithm $M$ is the pair of distributions $Q^t$ (the $t$-wise product distribution) and
\[
    \bar{P}_t=\frac{1}{t}\sum_{i\in [t]} Q^{i-1}\times P \times Q^{t-i},
\]
where $Q$ and $P$ are a dominating pair of distributions for $M$. Equivalently, we can reduce the analysis of a potentially very complicated algorithm like DP-SGD where steps can depend on the outputs of previous steps to the analysis of random allocation applied to a fixed randomizer (specifically, one that samples from a distribution $P$ when its input contains the user's data and samples from distribution $Q$ otherwise). As discussed in prior work, this reduction is tight for many private algorithms that include DP-SGD with sufficiently rich loss functions \citep{CGHLKKMSZ24,CCGHST25,FS25}.

Now, our goal is to compute the PLD, or the distribution of $\ln(\bar{P}_t(x)/Q^t(x))$ for $x\sim \bar{P}_t$. Somewhat more formally, we need to produce a sufficiently accurate upper bound on this random variable to allow computation of the privacy parameters for both directions of the divergence. In general, computing a PLD of a mixture of high-dimensional distributions is unlikely to be computationally tractable. Our main observation is that the PLD of the mixture arising in random allocation can be represented as the (logarithm of) a sum of exponentiated PLDs (or $\exp$-PLDs) of the dominating pair (see Theorem~\ref{thm:PLDrandAlloc} for a formal statement) and its inverse. We are dealing with an asymmetric add/remove notion of privacy, and therefore this result needs to be proved for both directions. This reduces our computation to $t$-wise convolutions performed on $\exp$-PLDs and their inverses.

We then describe how to appropriately discretize the $\exp$-PLDs and compute the $t$-wise convolutions for both directions in time logarithmic in $t$ and inverse quadratic in the desired accuracy (see Theorem~\ref{thm:num_acc_RA} for a formal statement). 
The dependence on accuracy is quadratic since ensuring tightness of our estimation requires multiplicative discretization of $\exp$-PLDs which is equivalent to the standard additive discretization of the PLD.  Such discretization has an extremely large dynamic range. As a result implementing FFT that relies on additive discretization requires a grid that is often too large for efficient computation\footnote{For comparison, FFT is the standard approach for computing the PLD of composition via convolution of PLDs since PLDs are discretized additively \citep{KJH20,KJPH21, KH21,GLW21}.}. At the same time, in some specific regimes, FFT over an additive grid might still be more efficient than the multiplicative grid approach. We discuss these computational considerations in Section \ref{sec:numRes}. The logarithmic dependence on $t$ is achieved using the standard exponentiation-by-squaring approach of doubling the number of steps via self-convolution and then using the binary representation of $t$ to compute the result of $t$-step convolution.

To compute an upper bound on the PLD for general $k$-out-of-$t$ random allocation, we use a slightly tighter variant of the reduction from $k$-out-of-$t$ random allocation to $k$ compositions of $1$-out-of-$t/k$ random allocation in \citet{FS25}, better accounting for the case where $t$ is not divisible by $k$ (Lemma \ref{lem:multAlloc}).

\para{Numerical evaluation:}
We compare our approach to existing techniques as well as Poisson subsampling in a variety of parameter settings.
While our technique is general, we focus our evaluation on Gaussian noise addition since it is the motivating application and the only case handled by most prior works. We note that we do not provide explicit results on the utility of random allocation, as such results can be found in prior work \citep{CGHLKKMSZ24,CCGHST25,FS25,DCO25,AFKRT26}. Our privacy bounds only require knowing the noise and sampling parameters used there. However, we repeat the privacy-utility trade-off toy experiment in \citep{FS25}, showing that random allocation improves on Poisson subsampling in this setting (Fig.~\ref{fig:utility_comparison}). Additional details of these numerical evaluations can be found in Section \ref{sec:numRes} and Appendix~\ref{apd:expRes}, where we provide an extensive comparison of Poisson and random allocation privacy bounds in various parameter regimes and for other local algorithms as well.

\begin{figure*}[t!]
    \centering\includegraphics[width=\linewidth]{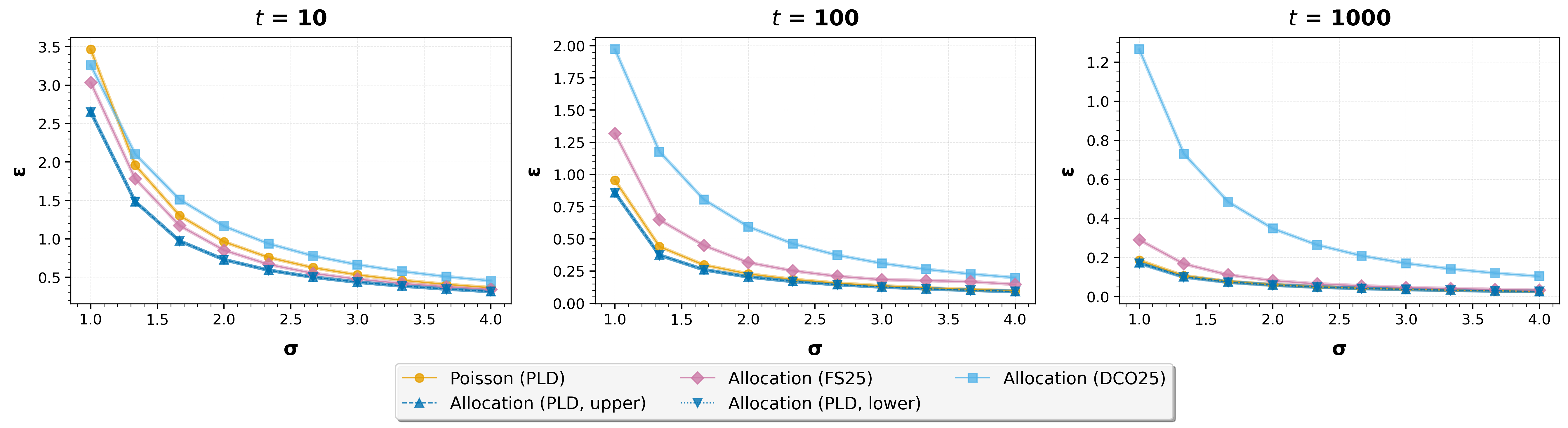}
    \caption{Upper and lower bounds on privacy parameter $\eps$ as a function of the noise parameter $\sigma$ for various values of $t$, all using the Gaussian mechanism with fixed $\delta = 10^{-6}$. We compare our upper and lower bounds (which are nearly identical) to upper bounds on random allocation \citep{FS25, DCO25}, and to the Poisson scheme with $\lambda = 1/t$.}
    \label{fig:epsilon_vs_sigma_by_t}
\end{figure*}

We start with a basic comparison with existing analysis methods for $k=1$ and a range of $t$ and $\sigma$ (Figure~\ref{fig:epsilon_vs_sigma_by_t}). 
As can be seen from the plots, our results improve upon all prior bounds and are never worse than the bounds for Poisson subsampling, which matches recent asymptotic analysis by \citet{DO26}. They also match those obtained via Monte Carlo simulations in the regimes where the latter produce reliable results (Figure \ref{fig:delta_comparison_partial}). We remark that the privacy bounds for these sampling techniques are generally incomparable (see Figure~\ref{fig:no-domination}).

\subsection{Related Work}
\begin{movable}{mov:relWork}
{
    Our work is most closely related to a long line of research on privacy amplification by subsampling and composition. This combination of tools was first defined and theoretically analyzed in the setting of convex optimization \citep{BST14}. Applications of DP-SGD in machine learning were spearheaded by the landmark work of  \citet{ACGMMT16} who significantly improved the privacy analysis via the moments accounting technique formalized via R\'{e}nyi DP \citep{Mironov17}. This work has also motivated the development of more advanced techniques for analysis of sampling and composition. A more detailed technical and historical overview of subsampling and composition for DP can be found in the survey by \citet{Steinke25}. 
    
    Our work closely aligns with the PLD-based numerical accounting techniques \citep{MM18, SMM19, KJH20,KJPH21, KH21, GLW21} which are now the standard approach for the analysis of DP-SGD supported by several libraries \citep{Google20,Microsoft21,Meta21}. We note that while our computation also involves convolutions, we are adding probability ratios and not their logarithms while ensuring the same kind of approximation guarantees. As a result, our algorithm is substantially different. At the same time, our algorithmic results fit naturally with the rest of the PLD toolkit and expand it to random allocation and general (single step) subsampling.

    Two recent works give formal analyses of $k$-out-of-$t$ random allocation \citep{FS25,DCO25}. \citet{FS25} describe three approximation approaches that are incomparable and also analyze the asymptotic behavior of random allocation, based on the hockey-stick and R\'{e}nyi divergences (DP and RDP). A RDP-based approach was independently proposed in \citep{DCO25}, and later computationally improved by \citet{SK26}. They provide upper bounds on the RDP parameters of the dominating pair of distributions in the Gaussian case for both add and remove directions. Methods based on RDP parameters are particularly well-suited for subsequent composition (which simply adds up the RDP parameters). The primary disadvantage of this technique is that the conversion from RDP bounds to the regular $(\eps,\delta)$ bounds is known to be somewhat lossy. The RDP bounds in \citep{FS25,DCO25,SK26}  are also harmed by the restriction $\alpha \ge 2$ since lower order $\alpha$ lead to the best $(\eps,\delta)$ parameters in some cases. An extended discussion of related work can be found in App.~\ref{apd:relWork}
}

    Our work is most closely related to a long line of research on privacy amplification by subsampling and composition. This combination of tools was first defined and theoretically analyzed in the setting of convex optimization \citep{BST14}. The resulting DP-SGD algorithm has found numerous applications in both theoretical and practical work and is currently the state-of-the-art method for training LLMs with provable privacy guarantees \citep{VaultGemma25}. Applications of DP-SGD in machine learning were spearheaded by the landmark work of \citet{ACGMMT16}, who significantly improved the privacy analysis via the moments accounting technique formalized via R\'{e}nyi DP \citep{Mironov17}. This work has also motivated the development of more advanced techniques for analysis of sampling and composition. A more detailed technical and historical overview of subsampling and composition for DP can be found in the survey by \citet{Steinke25}.

    One of the important tools that emerged for the analysis of DP-SGD is privacy accounting via numerical tracking of the privacy loss random variable. \citet{MM18, SMM19} coined the term \emph{privacy buckets} to describe the quantization of the PLD into a distribution over a finite set of values, which can then be numerically composed. Interestingly, they did not use FFT based composition due to numerical stability related challenges, and instead used direct numerical convolution as we do, as well as exponentiation by squaring for achieving logarithmic dependence of the runtime on the number of compositions. Later, \citet{KJH20,KJPH21, KH21} introduced an FFT based composition method, which significantly improved the runtime of the convolution for a given approximation target. The asymptotic dependence of the runtime on the desired approximation accuracy was tightly analyzed by \citet{GLW21}. This approach to composition improved on the moments accountant technique since it avoids the somewhat lossy conversion from RDP parameters to $(\eps,\delta)$ and is now the standard approach for the analysis of DP-SGD supported by several libraries \citep{Google20,Microsoft21,Meta21}. 
    
    We first note that while our computation also involves convolutions, we are adding probability ratios and not their logarithms while ensuring the same kind of approximation guarantees. As a result, our algorithm is substantially different. At the same time, our algorithmic results fit naturally with the rest of the PLD toolkit and expand it to random allocation and general (single step) subsampling.
    
    The shuffle model was first proposed by \citet{BEMMRLRKTS17}. The formal analysis of the privacy guarantees in this model was initiated in \citep{EFMRTT19, CSUZZ19}. The sequential shuffling scheme we discuss here was defined by \citet{EFMRTT19} who proved the first general privacy amplification results for this scheme, albeit only for pure DP algorithms. Improved analyses and extensions to approximate DP were given in \citep{BBGN19, BKMT20, FMT21, FMT23, GDDKS21, GDDSK21, KHH22}. The privacy amplification guarantees of shuffling also apply to $1$-out-of-$t$ random allocation. Indeed, random $1$-out-of-$t$ allocation is a special case of the {\em random check-in} model of defining batches for DP-SGD in \citep{BKMT20}. Their analysis of this variant relies on the amplification properties of shuffling and thus does not lead to better privacy guarantees for random allocation than those that are known for shuffling. 
    
    Two recent works give formal analyses of $k$-out-of-$t$ random allocation \citep{FS25,DCO25}. \citet{FS25} describe three approximation approaches that are incomparable and also analyze the asymptotic behavior of random allocation. In the first analysis, they show that the approximate DP $(\eps,\delta)$ privacy parameters of random allocation are upper bounded by those of the Poisson scheme with sampling probability $\approx k/t$ up to lower order terms which are asymptotically vanishing in $t/k$. This analysis does not lead to tight bounds when $t/k$ is small and can at best match the bounds for the Poisson sampling.  In the second analysis, they show that $\eps$ of random allocation with $k=1$ is at most a constant ($\approx 1.6$) factor times larger than $\eps$ of the Poisson sampling with rate $1/t$ for the same $\delta$.  This analysis gives better bounds for small $t$, but is typically worse by the said factor than Poisson sampling.
    
    \citet{FS25} also describe a direct analysis of the divergence for the dominating pair of distributions.  In the remove direction, they derive a closed form expression and relatively efficient algorithm for computing the integer $\alpha \geq 2$ order RDP parameters of random allocation in terms of the RDP parameters of the original algorithm. For the add direction, they give an approximate upper bound directly on the $(\eps,\delta)$ parameters. While this bound is approximate, the divergence for the add direction is typically significantly lower than the one for the remove direction and therefore even reasonably loose approximation of the add direction  tends to not harm the overall bound. A similar approach to the analysis of random allocation was independently proposed in \citep{DCO25}, and later computationally improved by \citet{SK26}. They provide upper bounds on the RDP parameters of the dominating pair of distributions in the Gaussian case for both add and remove directions.  Their efficiently computable bound is exact for $\alpha=2$ for the add direction and general $k$ and is approximate otherwise.
    
    Methods based on RDP parameters are particularly well-suited for subsequent composition (which simply adds up the RDP parameters). The primary disadvantage of this technique is that the conversion from RDP bounds to the regular $(\eps,\delta)$ bounds is known to be somewhat lossy (typically within $10$-$20\%$ range in multi-epoch settings). The bounds in \citep{FS25,DCO25,SK26}  are also harmed by the restriction $\alpha \ge 2$ since lower order $\alpha$ lead to the best $(\eps,\delta)$ parameters in some cases.
\end{movable}

\section{Preliminaries}
\begin{movable}{mov:prelim}
{
    We briefly present the definitions and notation that we use in this work. Detailed preliminaries with formal definitions can be found in Appendix \ref{apd:prelim}.
    In this work we consider \emph{$t$-step algorithms} defined using an algorithm $\alg$ that takes some subset of the dataset and a sequence of previous outputs\textemdash{}\emph{views}\textemdash{}as an input. We denote the domain of its input \emph{elements} by $\domain$ and the set of its possible \emph{outputs} by $\outDom$.
    Such algorithms include DP-SGD, where each step consists of a call to the Gaussian noise addition with (clipped) gradient vectors adaptively defined as a function of previous outputs.
    In this work, we consider the \emph{random allocation scheme} parametrized by a number of selected steps $k \in [t]$, which uniformly samples a set of $k$ distinct indices $i_{1}, \ldots, i_{k} \in [t]$ for each element $j$ and adds element $j$ to the corresponding subsets $\dataset^{i_{1}}, \ldots, \dataset^{i_{k}}$.
    For a $t$-step algorithm defined by an algorithm $\alg$, we denote by $\alloc{t, k}{\alg} : \domain^{*} \to \outDom^{t}$ the resulting algorithm when $\alg$ is used with the random allocation scheme. When $k = 1$ we omit it from the notation for clarity.
    
    We focus on the description that is based on the distribution of the privacy loss random variable \citep{DR16} and the closely related notion of the hockey-stick privacy profile \citep{BBG18}.
    Given two distributions $P, Q$, the \emph{privacy loss random variable} $\lossRVfunc{P}{Q}$ is defined by $\lossFunc{\omega}{P}{Q} \coloneqq \ln \left(\frac{P(\omega)}{Q(\omega)}\right)$ where $\omega \sim P$. We refer to its distribution as the \emph{privacy loss distribution (PLD)}.
    For $\kappa \in [0, \infty]$, the $\kappa$-\emph{hockey-stick divergence} between two distributions $P, Q$ is defined as 
    \[
        \HockeyStickDiv{\kappa}{P}{Q} \coloneqq \expect{}{\left[1 - \kappa \cdot e^{-\lossRVfunc{P}{Q}} \right]_{+}},
    \]
    where $\left[x \right]_{+} \coloneqq \max\{0, x\}$ \cite{BKOZB12}.
    We note that this definition extends to any random variable $L$, that is, $\HockeyStickFunc{\kappa}{\lossRV} \coloneqq \expect{}{\left[1 - \kappa \cdot e^{-\lossRV} \right]_{+}}$.
    
    The dataset adjacency notion we consider is the standard add/remove notion in which datasets $\dataset, \dataset' \in \domain^{*}$ are adjacent if $\dataset$ can be obtained from $\dataset'$ via adding or removing a single element. 
    Given an algorithm $\alg : \domain^{*} \times \outDom^{*} \to \outDom$, the privacy profile $\privProf{\alg} : \reals \to [0, 1]$ is defined to be the maximal hockey-stick divergence between the distributions induced by any adjacent datasets and past view. 
    Since the hockey-stick divergence is asymmetric in the general case, we use $\privProfRem{\alg}$ to denote the \emph{remove} direction where $\bot \in \dataset'$ and $\privProfAdd{\alg}$ to denote the \emph{add} direction when $\bot \in \dataset$. Consequently, $\privProf{\alg}(\eps) = \max\{\privProfRem{\alg}(\eps), \privProfAdd{\alg}(\eps) \}$. 
    Given $\eps > 0$; $\delta \in [0, 1 ]$, an algorithm $\alg$ will be called \emph{$(\eps, \delta)$-differentially private (DP)}, if $\privProf{\alg}(\eps) \le \delta$ \citep{DKMMN06}.
    
    A key concept for characterizing the privacy guarantees of an algorithm is that of a \emph{dominating pair} of distributions \citep{ZDW22}. Given distributions $P_{1}, Q_{1}, P_{2}, Q_{2}$, we say $\lossRVfunc{P_{1}}{Q_{1}}$ \emph{dominates} $\lossRVfunc{P_{2}}{Q_{2}}$ if for all $\kappa \ge 0$ we have $\HockeyStickDiv{\kappa}{P_{2}}{Q_{2}} \le \HockeyStickDiv{\kappa}{P_{1}}{Q_{1}}$, and if $\privProf{\alg}(\eps) \le \HockeyStickDiv{e^{\eps}}{P}{Q}$ for all $\eps \in \reals$, we say $(P, Q)$ is a \emph{dominating pair} of distributions for $\alg$.
    A dominating pair defines a dominating randomizer $\rand : \{\bot, *\} \rightarrow \outDom$, which captures the privacy guarantees of a $t$-step algorithm independently of its algorithmic adaptive properties, letting $\rand(*) = P$ denote the case where some specific data element was used by the algorithm and $\rand(\bot) = Q$ if not \citep{FS25}.
    
    The definition of the random allocation scheme naturally extends to the case where the internal algorithm is a randomizer, $\alloc{t, k}{\rand} : \{*, \bot\} \rightarrow \outDom^{t}$, and the domination of $\rand$ extends to this scheme (\ref{lem:allocRedRand}).
    Random allocation of a randomizer with $k > 1$ can be further reduced to a composition of single allocations, that is $\privProf{\alloc{t, k}{\rand}}(\eps) \le \privProf{\alloc{\lfloor t/k \rfloor}{\rand}}^{\otimes k}(\eps)$ where ${\otimes k}$ denotes the composition of $k$ runs of the algorithm \citep{FS25} (we use a slightly tighter variant that better handles $t$ not divisible by $k$, \ref{lem:multAlloc}).
    The dominating pair of distributions for random allocation with $k=1$ has a simple explicit form; $\allocFunc{t}{\alg}{\bot}$ is distributed as $Q^{t}$ and $\allocFunc{t}{\alg}{*}$ is distributed as  $\bar{P}_{t}$ where $Q^{t}$ is the product distribution of $t$ independent draws from $Q$ and $\bar{P}_t \coloneqq \frac{1}{t}\sum_{i\in [t]} Q^{i-1}\times P \times Q^{t-i}$ (\ref{clm:dominate-random-alloc-pair}).
    Combining these results, the analysis of general algorithms with random allocation scheme reduces to the analysis of the (composition of) random allocation of a single pair of distributions with a single allocation, which we do in the next sections.
}

    In this work we consider \emph{$t$-step algorithms} defined using an algorithm $\alg$ that takes some subset of the dataset and a sequence of previous outputs as an input. We denote the domain of its input \emph{elements} by $\domain$ and the set of its possible \emph{outputs} by $\outDom$. Formally, denoting $\outDom^{<t} = \bigcup_{i<t} \outDom^{i}$, $\alg$ takes a dataset in $\domain^{*}$ and a view $\view \in\outDom^{<t}$ as its inputs, and outputs a value in $\outDom$. A $t$-step algorithm using $\alg$ first uses some scheme to define $t$ subsets $\dataset^{1}, \ldots, \dataset^{t} \subseteq \dataset$, then sequentially computes $\out_{i} = \alg \left(\dataset^{i}, \view_{1:i-1} \right)$, where $\view_{1:i} \coloneqq (\out_{1}, \ldots, \out_{i})$ are the intermediate \emph{views} consisting of the outputs produced so far, and $\view_{1:0} = \emptyset$.
    Such algorithms include DP-SGD, where each step consists of a call to the Gaussian noise addition with (clipped) gradient vectors adaptively defined as a function of previous outputs.
    
    The assignment of the elements in $\dataset$ to the various subsets can be done in a deterministic manner (e.g., $\dataset^{1} = \ldots = \dataset^{t} = \dataset$), or randomly using a \emph{sampling scheme}. In this work, we consider the \emph{random allocation scheme} parametrized by a number of selected steps $k \in [t]$, which uniformly samples a set of $k$ distinct indices $i_{1}, \ldots, i_{k} \in [t]$ for each element $j$ and adds element $j$ to the corresponding subsets $\dataset^{i_{1}}, \ldots, \dataset^{i_{k}}$.
    For a $t$-step algorithm defined by an algorithm $\alg$, we denote by $\alloc{t, k}{\alg} : \domain^{*} \to \outDom^{t}$ the resulting algorithm when $\alg$ is used with the random allocation scheme. When $k = 1$ we omit it from the notation for clarity. We use the \emph{Poisson scheme}, which includes each element in each subset with probability $k/t$ independently of other elements and subsets, as the baseline for our numerical comparison.
    
    Given a random variable $X$ we denote its PDF by $f_{X}$, its CDF by $F_{X}$, and its complementary CDF (CCDF) by $\bar{F}_{X} = 1 - F_{X}$. We also consider random variables with $\pm \infty$ values.
    
    \subsection{Differential privacy and Privacy loss distribution}
    The privacy properties of a differentially private algorithm can be described in a number of different ways (see \citep{CJS25} for an overview of the relationships).  Here we focus on the description that is based on the distribution of the privacy loss random variable \citep{DR16} and the closely related notion of the hockey-stick privacy profile \citep{BBG18}.
    The privacy loss random variable can be defined for an arbitrary pair of distributions $P$ and $Q$, but it is typically instantiated with $P$ and $Q$ being the output distributions of an algorithm on two adjacent datasets.
    \begin{definition}[PLD \citep{DR16}]\label{def:PLD}
        Given two distributions $P, Q$ over some domain $\Omega$, the \emph{privacy loss random variable} $\lossRVfunc{P}{Q}$ is defined by $\lossFunc{\omega}{P}{Q} \coloneqq \ln \left(\frac{P(\omega)}{Q(\omega)}\right)$ where $\omega \sim P$. We refer to its distribution as the \emph{privacy loss distribution (PLD)}.
    \end{definition}
    
    Given the PLD we can define the standard hockey-stick divergence between distributions and extend it to random variables more generally.
    \begin{definition}[Hockey-stick divergence \cite{BKOZB12}]\label{def:HSdiv}
        Given $\kappa \in [0, \infty]$, the $\kappa$-\emph{hockey-stick divergence} between two distributions $P, Q$ is defined as 
        \[
            \HockeyStickDiv{\kappa}{P}{Q} \coloneqq \int_{\Omega} \left[P(\omega) - \kappa Q(\omega) \right]_{+} d\omega = \expect{}{\left[1 - \kappa \cdot e^{-\lossRVfunc{P}{Q}} \right]_{+}},
        \]
        where $\left[x \right]_{+} \coloneqq \max\{0, x\}$.
        We note that this definition extends to any random variable $L$ defining its $\kappa$-\emph{hockey-stick functional} as $\HockeyStickFunc{\kappa}{L} \coloneqq \expect{}{\left[1 - \kappa \cdot e^{-L} \right]_{+}}$.
    \end{definition}
    
    The dataset adjacency notion we consider is the standard add/remove notion in which datasets $\dataset, \dataset' \in \domain^{*}$ are adjacent if $\dataset$ can be obtained from $\dataset'$ via adding or removing a single element. To define sampling schemes that operate over a fixed number of elements appropriately, we augment the domain with a ``null'' element $\bot$, that is, we define $\domain' \coloneqq \domain \cup \{\bot\}$.
    When a $t$-step algorithm assigns $\bot$ to $\alg$ we treat it as an empty set, that is, for any $\dataset \in \domain^{*}$, $\view \in \outDom^{*}$ we have $\alg(\dataset, \view) = \alg((\dataset, \bot), \view)$.
    We say that two datasets $\dataset, \dataset' \in \domain^{n}$ are \emph{adjacent} and denote it by $\dataset \simeq \dataset'$, if one of the two can be created by replacing a single element in the other dataset by $\bot$.
    
    Using this notion, we define the privacy profile of an algorithm and use it to define differential privacy.
    \begin{definition}[Privacy profile \citep{BBG18}]\label{def:privProf}
        Given an algorithm $\alg : \domain^{*} \times \outDom^{*} \to \outDom$, the privacy profile $\privProf{\alg} : \reals \to [0, 1]$ is defined to be the maximal hockey-stick divergence between the distributions induced by any adjacent datasets and past view. Formally, 
        \[
            \privProf{\alg}(\eps) \coloneqq \underset{\dataset \simeq \dataset' \in \domain^{*}, \view \in \outDom^{*}}{\sup} \left( \HockeyStickDiv{e^{\eps}}{\alg(\dataset, \view) }{ \alg(\dataset', \view)} \right).
        \]
        Since the hockey-stick divergence is asymmetric in the general case, we use $\privProfRem{\alg}$ to denote the \emph{remove} direction where $\bot \in \dataset'$ and $\privProfAdd{\alg}$ to denote the \emph{add} direction when $\bot \in \dataset$. Consequently, $\privProf{\alg}(\eps) = \max\{\privProfRem{\alg}(\eps), \privProfAdd{\alg}(\eps) \}$. 
    \end{definition}
    
    We can now formally define the standard notion of DP.
    \begin{definition}[Differential privacy \citep{DKMMN06}]
        Given $\eps > 0$; $\delta \in [0, 1 ]$, an algorithm $\alg$ will be called \emph{$(\eps, \delta)$-differentially private (DP)}, if $\privProf{\alg}(\eps) \le \delta$.
    \end{definition}
    
    One of the most common DP algorithms is the Gaussian mechanism $N_{\sigma}$. This algorithm is defined using some function $f \colon \domain^* \to \reals^d$ of bounded $L_{2}$ sensitivity. Namely, for any pair of adjacent datasets $\dataset \simeq \dataset'$, $\|f(\dataset)-f(\dataset')\|_{2}\leq c$ for some value $c$. For example, in DP-SGD, $f$ is the average of norm-clipped gradients on all elements in $\dataset$. For the noise scale $\sigma$, the mechanism outputs the value of $f$ on the input dataset perturbed by Gaussian noise scaled to the sensitivity, namely a random sample from $\mathcal{N}(f(\dataset), c^2 \sigma^{2} I_{d})$.
    
    \subsection{Dominating Pair for Random Allocation}
    A key concept for characterizing the privacy guarantees of an algorithm is that of a \emph{dominating pair} of distributions \citep{ZDW22}. 
    \begin{definition}[Dominating pair \citep{ZDW22}] \label{def:domPair}
        Given distributions $P_{1}, Q_{1}$ over some domain $\Omega_{1}$, and $P_{2}, Q_{2}$ over $\Omega_{2}$, we say $\lossRVfunc{P_{1}}{Q_{1}}$ \emph{dominates} $\lossRVfunc{P_{2}}{Q_{2}}$ if for all $\kappa \ge 0$ we have $\HockeyStickDiv{\kappa}{P_{2}}{Q_{2}} \le \HockeyStickDiv{\kappa}{P_{1}}{Q_{1}}$.
         If $\privProfRem{\alg}(\eps) \le \HockeyStickDiv{e^{\eps}}{P}{Q}$ for all $\eps \in \reals$, we say $(P, Q)$ is a \emph{dominating pair} of distributions for $\alg$ in the remove direction. Similarly, if $\privProfAdd{\alg}(\eps) \le \HockeyStickDiv{e^{\eps}}{P}{Q}$ for all $\eps \in \reals$, we say $(P, Q)$ is a \emph{dominating pair} of distributions for $\alg$ in the add direction.
    \end{definition}
    
    For example, a dominating pair of distributions for the Gaussian mechanism $N_{\sigma}$ is simply $\mathcal{N}(1, \sigma^{2})$ and $\mathcal{N}(0, \sigma^{2})$ \citep{ZDW22}, so the privacy loss of $N_{\sigma}$ is $\lossFunc{x}{\mathcal{N}(1, \sigma^{2})}{\mathcal{N}(0, \sigma^{2})} = \frac{x}{\sigma^{2}} - \frac{1}{2\sigma^{2}}$ and its PLD is simply $\mathcal{N}\left(\frac{1}{2 \sigma^{2}}, \frac{1}{\sigma^{2}} \right)$.
    
    The notion of dominating pair can be used to define a dominating randomizer, which captures the privacy guarantees of a $t$-step algorithm independently of its algorithmic adaptive properties. 
    \begin{definition}[Dominating randomizer]\label{def:domRand}
        Let $\rand : \{\bot, *\} \rightarrow \outDom$ be a randomizer and let $P$ and $Q$ denote its output distributions on $*$ and $\bot$, respectively. We say that a $t$-step algorithm $\alg \colon \domain^* \times \outDom^{<t} \to \outDom$ is \emph{dominated by the randomizer} $\rand$ if $(P, Q)$ are a dominating pair of distributions for $\alg(\cdot,\cdot)$ w.r.t. the remove direction over all indexes $i \in [t]$ and input partial views $\view_{1:i-1}$.
    \end{definition}
    When there exists a pair of datasets $\dataset \simeq \dataset'$ and a view $\view \in \outDom^{*}$ such that $P = \alg(\dataset, \view)$, $Q = \alg(\dataset', \view)$ then the privacy profile of $\rand$ is identical to that of $\alg$.
    
    The definition of the random allocation scheme naturally extends to the case where the internal algorithm is a randomizer, $\alloc{t, k}{\rand} : \{*, \bot\} \rightarrow \outDom^{t}$, and the domination of $\rand$ extends to this scheme, as formalized in the next claim.
    
    \begin{lemma}[Allocation reduction to randomizer \citep{FS25}]\label{lem:allocRedRand}
        Given $t \in \naturals$; $k \in [t]$ and an algorithm $\alg$ dominated by a randomizer $\rand$, we have $\privProfRem{\alloc{t, k}{\alg}}(\eps) \le \privProfRem{\alloc{t, k}{\rand}}(\eps)$ and $\privProfAdd{\alloc{t, k}{\alg}}(\eps) \le \privProfAdd{\alloc{t, k}{\rand}}(\eps)$ for all $\eps \in \reals$.
    \end{lemma}
    A special case of this result for Gaussian noise addition can also be found in \citep{CGHLKKMSZ24,CCGHST25,DCO25}. The domination given by this reduction is tight for many natural choices of $\alg$ whenever $\rand$ is the tightly dominating randomizer for $\alg$ (see \citep{FS25} for a more detailed discussion).
    
    Random allocation of a randomizer with $k > 1$ can be further reduced to a composition of single allocations. \citet{FS25} proved that the privacy profile of the $k$-out-of-$t$ random allocation is upper bounded by that of $1$-out-of-$\lfloor t/k \rfloor$ random allocation self-composed $k$ times. If $t$ is not divisible by $k$ this reduction is somewhat lossy. To overcome it we prove a slightly tighter variant of this lemma.
    \begin{lemma}[Reduction to a single allocation]\label{lem:multAlloc}
        For any $t \in \naturals$ and $k \in [t]$, we have $\privProf{\alloc{t, k}{\rand}}(\eps) \le \privProf{\mathcal{A}^{f}(\rand) \otimes \mathcal{A}^{c}(\rand)}(\eps)$ for all $\eps \in \reals$, where $\mathcal{A}^{f}(\rand)$ ($\mathcal{A}^{c}(\rand)$) denotes the composition of $m_{f}$ ($m_{c}$) runs of the random allocation scheme with $\lfloor t/k \rfloor$ ($\lceil t/k \rceil$) steps, and $m_{f}, m_{c}$ are the solutions of the equations $m_{f} + m_{c} = k\,,\, m_{f} \cdot \lfloor t/k \rfloor + m_{c} \cdot \lceil t/k \rceil = t$.
    \end{lemma}
    \begin{proof}
        The proof is identical to that of Lemma 3.2 in \citep{FS25}, using a more refined decomposition of the random allocation of $k$ indexes out of $t$ into a two-step process: first randomly split $t$ into $m_{f}$ subsets of size $\lfloor t/k \rfloor$ and $m_{c}$ subsets of size $\lceil t/k \rceil$, then run $\alloc{\lfloor t/k \rfloor, 1}{\rand}$ on each of the $m_{f}$ copies of the scheme and $\alloc{\lceil t/k \rceil, 1}{\rand}$ on each of the $m_{c}$ copies.
    \end{proof}
    
    While this reduction may still be somewhat lossy, we remark that an analogous reduction for Poisson sampling is exact. Namely, sampling independently at the rate of $k/t$ for $t$ steps is equivalent to sampling at the rate of $k/t$ for $t/k$ steps (which is the analog of $1$-out-of-$t/k$ random allocation) composed $k$ times. Thus, this reduction implies that the relationship between $k$-out-of-$t$ random allocation and $t$ rounds of $k/t$-rate Poisson subsampling is the same as the relationship between $1$-out-of-$t/k$ random allocation and $t/k$ rounds of $1$-out-of-$t/k$ Poisson subsampling whenever $t$ is divisible by $k$.
    
    The dominating pair of distributions for random allocation with $k=1$ has a simple explicit form \citep{FS25} (for the special case of Gaussian noise addition this result can also be found in \citep{CGHLKKMSZ24,DCO25}).
    \begin{claim}[Dominating pair of distributions for random allocation \citep{FS25}]
    	\label{clm:dominate-random-alloc-pair}
        Given an algorithm $\alg$ dominated by a randomizer $\rand$ as in Defn.~\ref{def:domRand}, we denote by $Q^{t}$ the product distribution of $t$ independent draws from $Q$ and by $\bar{P}_t \coloneqq \frac{1}{t}\sum_{i\in [t]} Q^{i-1}\times P \times Q^{t-i}$. Then we have that $\allocFunc{t}{\alg}{\bot}$ is distributed as $Q^{t}$ and $\allocFunc{t}{\alg}{*}$ is distributed as  $\bar{P}_{t}$, which implies that for any $\eps > 0$
        \[
            \privProfRem{\alloc{t}{\alg}}(\eps) \le \privProfRem{\alloc{t}{\rand}}(\eps) = \HockeyStickDiv{e^{\eps}}{\bar{P}_{t}}{Q^{t}} \quad \text{and} \quad \privProfAdd{\alloc{t}{\alg}}(\eps) \le \privProfAdd{\alloc{t}{\rand}}(\eps) = \HockeyStickDiv{e^{\eps}}{Q^{t}}{\bar{P}_{t}}.
        \]
    \end{claim}
    
    Combining these results, the analysis of general algorithms with random allocation scheme reduces to the analysis of the (composition of) random allocation of a single pair of distributions with a single allocation, which we do in the next sections. Unless specified otherwise, random variables are assumed to be independent.
\end{movable}

\section{PLD-Based Privacy Accounting}\label{sec:PLDreal}
One of the key properties of domination in hockey stick divergence (via the dominating pair) is that such domination is maintained under composition and subsampling. This property ensures that in order to compute a valid upper bound on the privacy profile of a complex algorithm like DP-SGD it suffices to compute the hockey stick divergence of a single pair of distributions. At the same time, composition, which is the main operation of privacy accounting, is significantly easier to compute numerically using operations on the PLD of the dominating pair of each step since composition corresponds to addition of the privacy losses, or, equivalently, convolution of PLDs. This numerical accounting
relies on maintaining a discrete approximation of the PMF of the privacy loss.

Discrete approximations of PLDs are also central to our numerical computation of the privacy profile of random allocation. In this section, we demonstrate that such PLD-based representations can be used for privacy accounting beyond composition. The key property of a discrete representation we introduce for this purpose is that of a {\em PLD realization}. Namely, a random variable $X$ represented by a discrete PMF is a PLD realization if it corresponds to the privacy loss random variable for some pair of distributions $P_X,Q_X$. The goal in PLD-based accounting is to compute a PLD-realization $X$ such that $P_X,Q_X$ dominates the analyzed algorithm. To achieve this we ensure that accounting steps performed on PLD realizations preserve this property. 

We describe how to perform the basic (Poisson) subsampling operation directly on a PLD realization in a way that preserves domination. Our algorithm for approximately computing a PLD of random allocation can also be seen as a domination-preserving computation on a PLD realization. 
Altogether, we provide the first accurate numerical privacy accounting for a class of algorithms that includes subsampling and random allocation, in addition to composition. In particular, it implies that accounting for DP-SGD can be done for any noise distribution for which a valid dominating PLD realization can be constructed, whereas existing libraries rely on using an analytic expression of the PLD of a (Poisson) subsampled Gaussian or Laplace mechanism \citep{Google20,Microsoft21,Meta21}. We use this PLD-based accounting method to obtain improved privacy accounting for the DP-SGD algorithm in \citep{AFKRT26}, where the noise distribution itself results from random allocation applied to the Gaussian mechanism (Fig. \ref{fig:PREAMBLE}). All proofs and additional claims can be found in Appendix \ref{apd:PLDreal}.

We now define PLD realization formally.

\begin{definition}\Arxiv{[PLD realization]}\label{def:PLDreal}
    A random variable $\lossRV$ over $[-\infty, \infty]$ is a \emph{PLD realization} if $\mathbb{E}\left[e^{-\lossRV}\right] \le 1$ and $f_{\lossRV}(-\infty)=0$.

    Given a PLD realization $\lossRV$ with measure $f_{\lossRV}$ its \emph{PLD dual} is the random variable $\dual{\lossRV}$ defined by the measure $f_{\dual{\lossRV}}(l) = f_{\lossRV}(-l) \cdot e^{l}$ with probability atom at $\infty$ defined as $f_{\dual{\lossRV}}(\infty) \coloneqq 1 - \mathbb{E}\left[e^{-\lossRV}\right]$.
\end{definition}
It is not hard to show that $(1)$ $\lossRVfunc{P}{Q}$ is always a PLD realization (\ref{clm:PLDisRealz}); $(2)$ if $\lossRV$ is a PLD realization, then $f_{\lossRV}$ is the PLD for the pair of distributions $f_{\lossRV}$ and $f_{-\dual{\lossRV}}$; and $(3)$ $\dual{\lossRV}$ is a PLD realization as well (\ref{clm:PLDdual}). From the definition, the PLD dual of $\lossRVfunc{P}{Q}$ is $\lossRVfunc{Q}{P}$. It is known that if $\lossRV$ dominates $\lossRV'$ then $\dual{\lossRV}$ dominates $\dual{\lossRV'}$ as well \citep[Lemma 28]{ZDW22}. The term dual PLD follows \citep[Definition 3]{SMM19}. The relation between $\lossRVfunc{P}{Q}$ and $\lossRVfunc{Q}{P}$ was pointed out in \citep[Remark 3.4]{GLW21}.

In practical terms, these facts imply that the privacy profile of any algorithm can be upper-bounded using a single random variable, which represents a pair of distributions that dominate that algorithm. Since domination is transitive, any dominating transformation of the random variable is a valid dominating pair (albeit, not necessarily tight).

\ifconferencemode
\else
    We first recall that composition can be viewed as a domination-preserving operation on PLD realization. The dominating pair for (adaptive) composition of two algorithms with dominating pairs $(P_{1},Q_{1})$ and $(P_{2},Q_{2})$ is $(P_{1} \times P_{2}, Q_{1}\times Q_{2})$ \citep[Theorem 10]{ZDW22}, which implies domination preserved under convolution. The PLD of $(P_{1} \times P_{2}, Q_{1}\times Q_{2})$ is just the sum of the individual PLDs: $\lossRVfunc{P_{1} \times P_{2}}{Q_{1} \times Q_{2}} = \lossRVfunc{P_{1}}{Q_{1}} + \lossRVfunc{P_{2}}{Q_{2}}$, therefore, given PLD realizations $\lossRV_{1}$ and $\lossRV_{2}$ that dominate $\lossRVfunc{P_{1}}{Q_{1}}$ and $\lossRVfunc{P_{2}}{Q_{2}}$, respectively, we get that $\lossRV_{1}+\lossRV_{2}$ is a PLD realization that dominates $\lossRVfunc{P_{1} \times P_{2}}{Q_{1} \times Q_{2}}$.
\fi
\begin{movable}{mov:subsamPLDdom}{}
    It is known that the dominating pair of distributions for Poisson subsampling is obtained by taking the corresponding convex combination of the dominating pair of distributions for the original algorithm.\footnote{Similar bounds can be derived for sampling a fixed number of elements with and without replacement following the same analysis.} 
    \begin{lemma}[\Arxiv{Theorem 11 in }\citep{ZDW22}]\label{lem:subsamPLDdom}
        Given $\lambda \in (0,1]$ and an algorithm $\alg$, denote by $\alg \circ \mathcal{P}$ the algorithm that, given a dataset $\dataset$, constructs a subset $\dataset'$ by independently including each element in $\dataset$ with probability $\lambda$ and then releases $\alg(\dataset')$. If $\alg$ is dominated by the pair of distributions $(P, Q)$ in the remove direction, then $\alg \circ \mathcal{P}$ is dominated by $(P_{\lambda}, Q)$ in the remove direction and by $(Q, P_{\lambda})$ in the add direction, where $P_{\lambda} \coloneqq \lambda P + (1-\lambda) Q$.
    \end{lemma}    
\end{movable}

We now show how the subsampling operation transforms the PLD realization itself.
\begin{theorem}\label{thm:PLD_subsam}
    Given $\lambda \in (0,1]$ and a PLD realization $\lossRV$, we define the two transformed random variables $\rem{\varphi}_{\lambda}(\lossRV), \add{\varphi}_{\lambda}(\lossRV)$ via their PMF,
    \ifconferencemode
\else
    for any $l$ (with the convention that a PMF is $0$ where $\phi_{\lambda}$ is undefined)
    \[
        f_{\rem{\varphi}_{\lambda}(\lossRV)}(l) \coloneqq \lambda f_{\lossRV}\left(\phi_{\lambda}(l)\right) + (1-\lambda) f_{-\dual{\lossRV}} \left(\phi_{\lambda}(l)\right), \quad f_{\add{\varphi}_{\lambda}(\lossRV)}(l) \coloneqq f_{\lossRV}\left(-\phi_{\lambda}(-l)\right), \quad\text{and}\quad \phi_{\lambda}(l) \coloneqq \ln(1 + (e^{l}-1) / \lambda).
    \]
    The map $\phi_{\lambda}(l) = \ln(1+(e^{l}-1)/\lambda)$ is defined only for $l \ge \ln(1-\lambda)$, so $\rem{\varphi}_{\lambda}(\lossRV)$ is supported on $[\ln(1-\lambda), \infty]$ (with an atom at $\ln(1-\lambda)$ collecting the mass mapped from $-\infty$) and $\add{\varphi}_{\lambda}(\lossRV)$ is supported on $[-\infty, -\ln(1-\lambda)]$; both PMFs are $0$ outside these ranges.
    \fi
    Given two distributions $P, Q$ we have, $\lossRVfunc{P_{\lambda}}{Q} = \rem{\varphi}_{\lambda}(\lossRVfunc{P}{Q})$ and $\lossRVfunc{Q}{P_{\lambda}} = \add{\varphi}_{\lambda}(\lossRVfunc{Q}{P})$.
\end{theorem}
The definition of the transformation over PLD realizations results from a direct analysis of the distribution, similar in spirit to the hockey-stick transformation in \citep[Theorem 2]{BBG18}. By explicitly writing the privacy loss between a mixture of $P$ and $Q$ and one of its components in terms of the privacy loss between $P$ and $Q$, subsampling can be treated as a PLD transformation, and the amplification as its resulting improved hockey-stick functional.

\Arxiv{
    In PLD realization terms, Lemma \ref{lem:subsamPLDdom} essentially guarantees that domination is preserved under the subsampling transformation so if $\lossRV$ dominates $\lossRVfunc{P}{Q}$ then $\rem{\varphi}_{\lambda}(\lossRV)$ dominates $\lossRVfunc{P_{\lambda}}{Q}$, and if $\lossRV$ dominates $\lossRVfunc{Q}{P}$ then $\add{\varphi}_{\lambda}(\lossRV)$ dominates $\lossRVfunc{Q}{P_{\lambda}}$.
    While the privacy profile of both add and remove directions is captured by the random variable in a single direction \citep{CJS25}, the tightness of the domination (i.e., the induced error in $\eps$ and $\delta$ induced by the slackness of the bound) may be asymmetric, which is why most numerical accounting libraries keep track of separate bounds per direction. We express our result in both directions to accommodate this approach.
}

Theorem \ref{thm:PLD_subsam} directly implies a practical approach for computing an upper bound on the PLD of any subsampled algorithm $\alg$ dominated by a pair of distributions $(P, Q)$, in time linear in the size of the support of $\lossRVfunc{Q}{P}$. Given the PLD realization $\lossRVfunc{P}{Q}$, compute its dual $\lossRVfunc{Q}{P}$ (in the remove case), transform the relevant supports according to $\phi_{\lambda}$, and compute the probability mass under $\lossRVfunc{P}{Q}$ and $\lossRVfunc{Q}{P}$ (Alg. \ref{alg:PLD_subsam_rem}, \ref{alg:PLD_subsam_add}). The resulting random variable dominates the subsampled algorithm.

\section{PLD Estimation for Random Allocation}\label{sec:PLDconv}
In this section we derive the PLD of random allocation and show how to compute it numerically.
To reason about the validity and tightness of our bounds, we introduce another natural notion of domination for PLD realization\textemdash{}stochastic domination. We quantify the tightness of domination using approximation parameters $\alpha$ governing the shift in the privacy loss (corresponding to $\eps$), and $\beta$ governing the gap in probability (corresponding to $\delta$). All proofs and additional claims can be found in Appendix \ref{apd:PLDconv}.

\begin{definition}\Arxiv{[Stochastic Domination]}\label{def:stochDom}
    A random variable $\upBnd$ \emph{(first order) stochastically dominates} $\lowBnd$ (denoted by $\lowBnd \preceq \upBnd$), if $\bar{F}_{\upBnd}$ upper bounds $\bar{F}_{\lowBnd}$, that is $\forall x \in [-\infty, \infty]; \bar{F}_{\lowBnd}(x) \le \bar{F}_{\upBnd}(x)$\conf{, where $\bar{F}_{X} = 1 - F_{X}$ is the complementary CDF (CCDF) of $X$}. Given $\alpha \ge 0$; $\beta \in [0,1]$, we say $\upBnd$ \emph{$(\alpha, \beta)$-approximately stochastically dominates} $\lowBnd$ and denote it by $\lowBnd \preceq_{(\alpha, \beta)} \upBnd$, if $\forall x \in [-\infty, \infty]:~ \bar{F}_{\lowBnd}(x) \le \bar{F}_{\upBnd}(x-\alpha) + \beta$. We say $\upBnd$ $(\alpha, \beta)$-tightly stochastically dominates $\lowBnd$ if $\lowBnd \preceq \upBnd \preceq_{(\alpha, \beta)} \lowBnd$.
\end{definition}
\Arxiv{
    A closely related tightness notion appeared under the name \emph{coupling approximation} in \citep[Definition 5.1]{GLW21}, using an equivalent condition that there exists a coupling between $\upBnd$ and $\lowBnd$, such that $\vert \upBnd-\lowBnd \vert < \alpha$ with probability at least $1- \beta$. This condition is equivalent to the requirement $\lowBnd \preceq_{(\alpha, \beta)} \upBnd \preceq_{(\alpha, \beta)} \lowBnd$. Moving forward, we refer to domination in the hockey-stick sense (Defn. \ref{def:domPair}) simply as ``domination'' and specify stochasticity otherwise. We note that a random variable stochastically dominating a PLD realization is a PLD realization as well (\ref{clm:stochDomPLDreal}).
}
    
It is known that (approximate) stochastic domination implies (approximate) domination in the hockey-stick sense \citep{GLW21}, that is if $\lowBnd \preceq_{(\alpha, \beta)} \upBnd$, then $\HockeyStick{e^{\eps}}(\lowBnd) \le \HockeyStick{e^{\eps-\alpha}}(\upBnd) + \beta$.
\begin{movable}{mov:stochDomImpDom}{}
    \begin{claim}\label{clm:stochDomImpDom}
    	Given $\alpha \ge 0$; $\beta \in [0,1]$ and two random variables $\upBnd, \lowBnd$, if $\lowBnd \preceq_{(\alpha, \beta)} \upBnd$, then $\HockeyStick{e^{\eps}}(\lowBnd) \le \HockeyStick{e^{\eps-\alpha}}(\upBnd) + \beta$.
    \end{claim}
\end{movable}

Consequently, if $\lossRVfunc{P}{Q} \preceq_{(\alpha, \beta)} \lossRV$, then $\HockeyStickDiv{e^{\eps}}{P}{Q} \le \HockeyStick{e^{\eps-\alpha}}(\lossRV) + \beta$.
\Arxiv{Notably, the reverse is not true. \citet[Theorem 12]{CJS25} provide an example for a pair of distributions dominating another pair in terms of hockey-stick but not stochastically, and we show that stochastic domination is not maintained under subsampling and the dual transformations even for PLD realizations (\ref{clm:PLDdualDom}).}

\Arxiv{
    The composition of two PLD realizations can be computed using efficient convolution methods such as FFT. The next Claim quantifies the tightness of such operation.
}

\begin{movable}{mov:stochDomComp}{}
    \begin{claim}\label{clm:stochDomComp}
        Given random variables $\upBnd_{i}, \lowBnd_{i}$ for $i \in \{1,2\}$, if $\lowBnd_{i} \preceq_{(\alpha_{i}, \beta_{i})} \upBnd_{i}$, then $\lowBnd_{1} + \lowBnd_{2} \preceq_{(\alpha_{1} + \alpha_{2}, \beta_{1} + \beta_{2})} \upBnd_{1} + \upBnd_{2}$.
    \end{claim}
    Consequently, if $\alg_{i}$ is dominated by $\lossRVfunc{P_{i}}{Q_{i}}$ and $\lossRVfunc{P_{i}}{Q_{i}} \preceq_{(\alpha_{i}, \beta_{i})} \upBnd_{i}$, then $\lossRVfunc{P_{1}\times P_{2}}{Q_{1}\times Q_{2}} \preceq_{(\alpha_{1} + \alpha_{2}, \beta_{1} + \beta_{2})} \upBnd_{1} + \upBnd_{2}$. 
\end{movable}

Next we show how random allocation with $k=1$ can be viewed as a transformation operating over a PLD realization, in the form of a convolution of the exponentiated PLDs (or $\exp$-PLDs).
\begin{theorem}\label{thm:PLDrandAlloc}
    For $t \in \naturals$ and any PLD realization $\lossRV$, we define two transformations on $\lossRV$,
    \ifconferencemode
\else
        \[
            \rem{\psi}_{t}(\lossRV) \coloneqq \ln\left(\frac{1}{t} \left(e^{\lossRV_{0}} + \sum_{i=1}^{t-1} e^{-\dual{\lossRV_{i}}}\right)\right) \quad \text{and} \quad \quad \add{\psi}_{t}(\lossRV) \coloneqq -\ln\left(\frac{1}{t} \sum_{i=1}^{t} e^{-\lossRV_{i}}\right),
        \]
        
    \fi
     where $\lossRV_{0}, \lossRV_{1}, \ldots$ (and $\dual{\lossRV_{i}}$) denote independent copies of $\lossRV$ (and its PLD dual, respectively).
     
    Given two distributions $P, Q$ we have, $\lossRVfunc{\bar{P}_{t}}{Q^{t}} = \rem{\psi}_{t}(\lossRVfunc{P}{Q})$ and $\lossRVfunc{Q^{t}}{\bar{P}_{t}} = \add{\psi}_{t}(\lossRVfunc{Q}{P})$.
\end{theorem}
Similar to subsampling, this theorem states that an amplification operation can be represented as a transformation of the base PLD between two distributions.
This identity results from two facts. First, we note that for any $\bar{\omega}_{1:t} = (\omega_{1}, \ldots, \omega_{t}) \in \Omega^{t}$, we have $\lossFunc{\bar{\omega}_{1:t}}{\bar{P}_{t}}{Q^{t}} = \ln\left(\frac{1}{t} \sum_{i \in [t]} e^{\lossFunc{\omega_{i}}{P}{Q}}\right)$. In the case of $\lossRVfunc{Q^{t}}{\bar{P}_{t}}$ we have $\omega_{i} \sim Q$ for all $i$, but $\lossRVfunc{\bar{P}_{t}}{Q^{t}}$ is defined by $\omega_{i} \sim P$ for a single uniformly sampled index $i$ and $\omega_{j} \sim Q$ for all $j \ne i$. The second insight is that from symmetry, the privacy loss is identically distributed regardless of the index $i$ sampled from $P$, so we can fix $\omega_{1} \sim P$ and $\omega_{i} \sim Q$ for all $i > 1$. \Arxiv{We note that in the case of the Gaussian mechanism, the $\exp$-PLDs are simply the log-normal random variable, so the PLD of the random allocation is simply the (negative of the) logarithm of the sum of $t$ log-normals.}

\Arxiv{
This theorem provides the PLD realization transformation corresponding to Claim \ref{clm:dominate-random-alloc-pair}, and  Lemma \ref{lem:allocRedRand} essentially states that domination is preserved under this random allocation transformation, that is, if $\lossRV$ dominates $\lossRVfunc{P}{Q}$ then $\rem{\psi}_{t}(\lossRV)$ dominates $\lossRVfunc{\bar{P}_{t}}{Q^{t}}$, and if $\lossRV$ dominates $\lossRVfunc{Q}{P}$ then $\add{\psi}_{t}(\lossRV)$ dominates $\lossRVfunc{Q^{t}}{\bar{P}_{t}}$.
}

\ifconferencemode
\else
    The PDF of a sum of independent random variables can be computed via convolution of PDFs and therefore Theorem \ref{thm:PLDrandAlloc} directly implies an approach for computing an upper bound on the PLD of the random allocation scheme for any algorithm $\alg$ dominated by a pair of distributions $(P, Q)$. Given a random variable $\lossRV$ dominating the PLD, compute its dual $\dual{\lossRV}$, convolve the $\exp$-PLDs of $\lossRV$ and $\dual{\lossRV}$, and transform their grid back using (negative) log.
    
    Naturally, these computations cannot be performed exactly for arbitrary continuous random variables. However, we show that a tight upper bound can be computed efficiently. Our analysis relies on the following lemma showing that convolution of $\exp$-PLDs preserves the tightness of stochastic domination.
\fi
\begin{movable}{mov:mult_tight}{}
    \begin{lemma}\label{lem:mult_tight}
    	Given $\alpha, \beta_{1}, \beta_{2} > 0$, and random variables $\upBnd_{1}, \upBnd_{2}, \lowBnd_{1}, \lowBnd_{2}$, if $\lowBnd_{i} \preceq_{(\alpha, \beta_{i})} \upBnd_{i}$ for $i \in \{1, 2\}$, then $\ln\left(e^{\lowBnd_{1}}+e^{\lowBnd_{2}}\right) \preceq_{(\alpha, \beta_{1} + \beta_{2})} \ln\left(e^{\upBnd_{1}}+e^{\upBnd_{2}}\right)$.
    \end{lemma}
\end{movable}

This lemma implies that unlike the convolution used in composition (Claim~\ref{clm:stochDomComp}), the original discretization (or binning) error of $\lossRV$ does not increase under this convolution. This error can be controlled while ensuring a manageable upper bound on the number of bins in a standard way. Specifically, to get a PLD realization that $(\alpha,\beta)$-tightly stochastically dominates a PLD $L$ we first define a finite evenly-spaced grid of width $\alpha$ over the range $\left[q_{\beta}\left(L\right), q_{1-\beta}\left(L\right)\right]$, where $q_{a}(X)$ is the $a$ quantile of the random variable $X$. We then round up the values of the random variable to the grid points with all the points above $q_{1-\beta}\left(L\right)$ rounded to $\infty$.

Evenly spaced bins of the PLD correspond to constant \emph{ratio} of its exponent. This representation is not well-suited for FFT-based convolution which operates on an additive grid. Therefore our algorithm relies on direct numerical convolution of pairs of distributions.  The number of computed convolutions can be minimized using exponentiation by squaring. We also apply a discretization step after each convolution to ensure an upper bound on the number of bins that the algorithm maintains. Overall, this leads to the following algorithm (sketch).
Full implementation details can be found in Appendix \ref{apd:impDtls}.

\ifconferencemode
\else
    \para{Algorithm outline (Alg. \ref{alg:main_alg_rem}, \ref{alg:main_alg_add}).} Given target accuracy parameters $\alpha, \beta$ and PLD realizations $\rem{\lossRV}$ ($\add{\lossRV}$):
    \begin{itemize}
        \item Compute the dual $\dual{\rem{\lossRV}}$ of the input, and discrete random variables $\rem{\lossRV}'$ dominating $\rem{\lossRV}$ and $\dual{\rem{\lossRV}}'$ dominating $\dual{\rem{\lossRV}}$ for the remove direction ($\add{\lossRV}'$ dominating $\add{\lossRV}$ for add), by discarding the extreme $\beta' \coloneqq \frac{\beta}{t}$ quantiles on both ends, and discretizing to constant width $\alpha' \coloneqq \frac{\alpha}{2 \lceil\log_{2}(t)\rceil+1}$ (Alg. \ref{alg:disc_dist}).
        \item Compute the convolution of $e^{\rem{\lossRV}'}$ with $t-1$ copies of $e^{-\dual{\rem{\lossRV}}'}$ for the remove direction ($t$ copies of $e^{-\add{\lossRV}'}$ for add) in $\le 2 \lceil\log_{2}(t)\rceil$ convolution steps, using exponentiation by squaring (Alg. \ref{alg:self_conv}).
        \item The convolution is computed directly over all possible values (squaring the number of bins) and discretized into a new geometrically spaced grid with resolution $\alpha'$ between the minimal and maximal possible values, rounding values to the right to create a dominating RV for the remove direction (rounding to the left to create a dominated random variable for the add direction) (Alg. \ref{alg:conv}).
        \item Return $\rem{\lossRV}_{t}$\textemdash{}the logarithm of the convolved random variable divided by $t$ for the remove direction ($\add{\lossRV}_{t}$\textemdash{}the negation of that log for the add direction).
    \end{itemize}
\fi

We can now formally state the properties of this algorithm.
\begin{theorem}\label{thm:num_acc_RA}
    There exists an algorithm that receives $\alpha > 0$; $\beta \in [0,1]$; $t \in \naturals$, and two PLD realizations $\rem{\lossRV}, \add{\lossRV}$ as input, and returns two PLD realizations $\rem{\lossRV}_{t}, \add{\lossRV}_{t}$ such that:
    \ifconferencemode
\else
        \begin{enumerate}
            \item \textbf{(1) Validity:} $\rem{\lossRV}_{t}$ stochastically dominates $\rem{\psi}_{t}(\rem{\lossRV})$ and $\add{\lossRV}_{t}$ stochastically dominates $\add{\psi}_{t}(\add{\lossRV})$.
            \item \textbf{(2) Tightness:} These stochastic dominations are $(\alpha, \beta)$-tight.
            \item \textbf{(3) Computation complexity:} The runtime of the algorithm is $O\left(\left(\frac{\mathtt{IQR}_{\beta/t}}{\alpha}\right)^2 \cdot \log^{3}(t) \right)$,
        \end{enumerate}
    \fi
    where for any $\eta \in [0,1/2]$, $\mathtt{IQR}_{\eta} \coloneqq \max\{\mathtt{IQR}_{\eta}(\rem{\lossRV}), \mathtt{IQR}_{\eta}(\dual{\rem{\lossRV}}), \mathtt{IQR}_{\eta}(\add{\lossRV})\}$ and for any random variable $X$, $\mathtt{IQR}_{\eta}(X) \coloneqq F^{-1}_{X}(1-\eta) - F^{-1}_{X}(\eta)$ is the distance between the $\eta$ and $1-\eta$ quantiles of $X$.
\end{theorem}

In the case of the Gaussian mechanism with sensitivity $1$, $\mathtt{IQR}_{\beta/t} = O\left(\frac{\sqrt{\ln(t/\beta)}}{\sigma} \right)$. In particular, the runtime of the algorithm is $O\left(\frac{\log_{2}^{3}(t) \ln(t/\beta)}{\sigma^{2} \alpha^{2}} \right)$.

Combining this theorem with Theorem~\ref{thm:PLDrandAlloc}, the derivation of the dominating pair of distributions for random allocation (Claim \ref{clm:dominate-random-alloc-pair}), and the fact that stochastic domination implies domination in the hockey-stick sense (Claim \ref{clm:stochDomImpDom}), we get a computationally efficient transformation on PLD realizations for random allocation. The validity of this algorithm's output is maintained by re-discretizing in a domination-preserving manner, the tightness\textendash{}by accounting for the accumulated effect of all discretizations, and runtime\textendash{}by tracking the number of convolutions and the induced grid size.

\Arxiv{
    \begin{remark}\label{rem:lower}
        A nearly identical algorithm produces a numerical lower bound on the PLD of $\alg$, with the same guarantees, by switching domination direction in all steps. We use this option to construct the lower bound in Figures \ref{fig:epsilon_vs_sigma_by_t} and \ref{fig:epsilon_vs_sigma_by_k}, to demonstrate the tightness of our bounds.
    \end{remark}

    \para{Additional truncation.} To improve the efficiency of our algorithm we also truncate the convolved distribution to its $[q(\beta'), q(1-\beta')]$ quantiles after each squaring (while ensuring that both lower and upper bounds are valid). This step is useful since the convolved distribution is more concentrated than the original one. We additionally apply the Chernoff bound-based shrinking of the FFT range for the composition over $k$, as first proposed by \citet{KJPH21}.
 }

\section{Numerical Results}\label{sec:numRes}

\begin{movable}{mov:FFT}{}
    \para{FFT convolution.} While it is also possible to implement the convolution using FFT, it turns out that in many reasonable settings, it is hard to achieve high accuracy (in stochastic domination and the implied privacy evaluation) using reasonable computational resources. This is because FFT requires a constant discretization, and the convolution is carried over the exponent of the loss which implies a significantly larger range. For example, in the case of the Gaussian mechanism the PLD is a Gaussian random variable and its exponent is a lognormal, so setting $\sigma = 1$ and $\beta = 10^{-10}$ induces a width of $\approx 12.5$  on the discretized PLD and $\approx 950$ on its exponent.
    
    While we do not provide explicit tightness guarantees for this method, our validity analysis of the bounds it produces holds nevertheless, and it provides superior bounds in some extreme cases in the remove direction. The difference between the add and remove directions results from the fact FFT bins are evenly spaced in the exponent of the privacy loss space, which implies higher resolution for large positive losses and lower resolution for large negative losses. Consequently, the FFT yields a tighter privacy profile in the remove direction, which roughly corresponds to the right tail bound on the PLD, than in the add direction, which roughly corresponds to the left tail bound. We elaborate on these points in Appendix \ref{apd:expRes}, where we show that combining the two methods may lead to superior bounds (Fig. \ref{fig:FFT_geom_comb}).
    
    We note that combining the FFT-based and direct convolution methods on the same random variable can result in a tighter bound than either one (Claim \ref{clm:bndComb}).     
\end{movable}

\Arxiv{
    \para{Numerical stability.}
    We note that in practice, numerical stability affects probabilities close to machine accuracy ($\approx 10^{-15}$ for float64), which can be mitigated by using float128 at the cost of additional computation. Since these inaccuracies grow with the number of compositions, this requires careful implementation whenever $\delta \le 10^{-15} / t$. To mitigate this effect, we implemented various measures such as using logsf and logcdf to compute the $\beta'$ quantiles, and Kahan summation to reduce the risk of error accumulation.
}

\para{Privacy profile.} To emphasize the advantage of our numerical accounting method for the random allocation scheme, we compare it in Figure \ref{fig:delta_comparison_partial} to the combined analytic methods in \citep{FS25} (using the tightest bound over all the methods they provide), the Monte Carlo-based estimation of the privacy profile (both mean and high probability bounds), the lower bound by \citet{CGHLKKMSZ24}, and the numerical accounting of the Poisson sampling scheme. The chosen parameters match those used by \citet{CGHLKKMSZ24}, corresponding to their experimental setting. Results for additional parameter regimes and full experiment details can be found in Appendix \ref{apd:expRes}.

\begin{figure}[ht]
    \centering
    \includegraphics[width=0.9\linewidth]{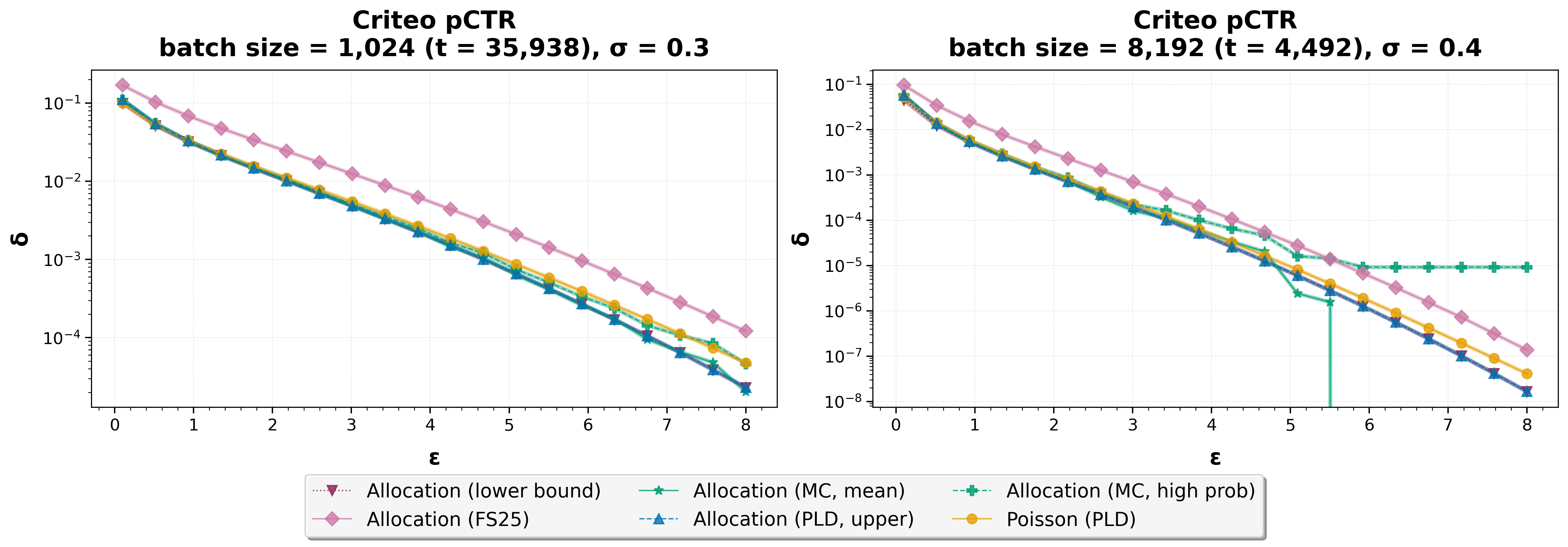}
    \caption{Comparison of the privacy profile of the Poisson scheme and various bounds for the random allocation scheme; the combined methods in \citet{FS25}, the high probability and the average estimations using Monte Carlo simulation and the lower bound by \citet{CGHLKKMSZ24}, and our numerical method, following the setting in \citet{CGHLKKMSZ24} (detailed description can be found in Appendix \ref{apd:expRes}).}
    \label{fig:delta_comparison_partial}
\end{figure}

Our results are nearly indistinguishable from the lower bound, closely match the expected MC-based estimation in the regime where it is statistically stable, are tighter than the MC-based high-probability bound for our simulation setting, and are significantly tighter than the analytical bounds. We also note that the privacy profile of the random allocation scheme is lower than that of Poisson for nearly the entire range. This is the case for nearly all parameter regimes (Fig. \ref{fig:param_scan}), except when $\eps \ll 1$ as depicted in Figure \ref{fig:no-domination}, matching the asymptotic analysis by \citet{DO26} for $\sigma \rightarrow 0$ and $\sigma \rightarrow \infty$, and the theoretical limits $\eps \rightarrow 0$ and $\eps \rightarrow \infty$ proven by \citet{CGHLKKMSZ24}.

We note that in the case of the Laplace mechanism the relation between Poisson and random allocation is more complex. Generally speaking, random allocation's privacy guarantees are better than Poisson's in the low privacy regime (small $\sigma$ and $t$), slightly worse than Poisson's in the high privacy regime, and converge as $\varepsilon$ tends to $0$ (Fig. \ref{fig:Gauss_Laplace_by_sigma}, \ref{fig:Gauss_Laplace_by_t}). A full characterization of the relation between Poisson and random allocation for various algorithms is left to future work.

\para{DP-SGD with PREAMBLE:}
To illustrate the accuracy of our PLD computations and our general approach to PLD-based accounting, we revisit the DP-SGD with low-communication noise addition algorithm (PREAMBLE) \citep{AFKRT26}. 
In this setting, a model is trained on a dataset of $n$ users for $E$ epochs, each consisting of $1/q$ gradient updates for some $q\in(0,1]$ via a variant of DP-SGD. Each update is calculated using a batch of users chosen via Poisson subsampling with probability $q$ using PREAMBLE mean estimation algorithm. In this algorithm the $d$-dimensional gradient is split into blocks of size $B$ and each user in the update batch samples $k = C/B$ (out of the $d/B$) blocks of the gradient. Here $C$ is an overall communication constraint for each user. Each user then adds Gaussian noise to their blocks and the blocks from all the users in the batch are sent to secret-sharing servers and aggregated.

In this algorithm each gradient update corresponds to running the $k = C/B$ out of $t = d/B$ random allocation scheme with Gaussian noise, followed by Poisson subsampling with rate $q$ and then $E / q$ composition. \citet{AFKRT26} used the RDP-based accounting derived from \citep{FS25,DCO25}, combined with amplification by subsampling for RDP and standard composition bounds.

\begin{figure}[ht]
    \centering
    \includegraphics[width=1.0\linewidth]{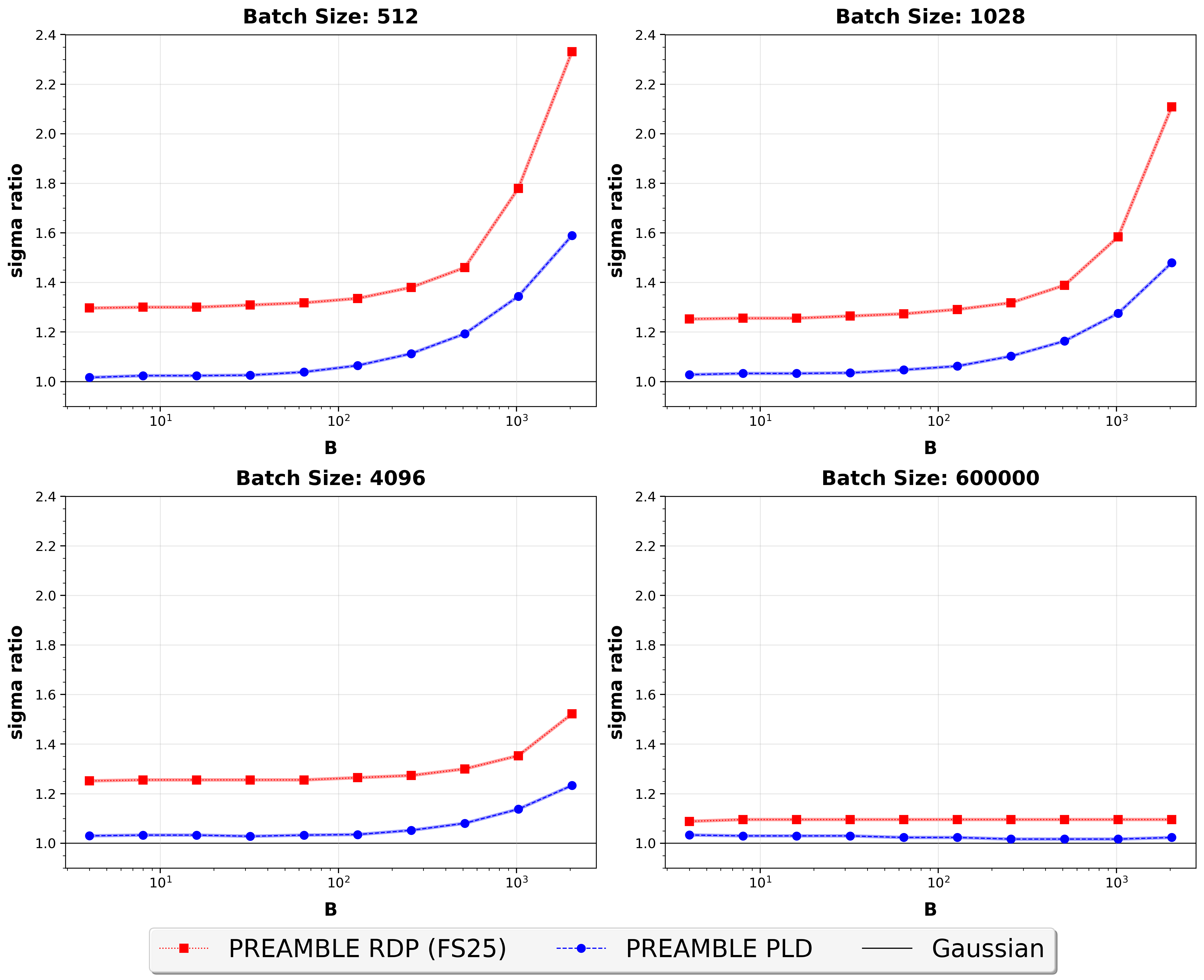}
    \caption{The ratio between the noise level required to achieve $(\eps=1, \delta=10^{-6})$-DP using the PREAMBLE method and simple Gaussian noise addition for DP-SGD calculated via RDP (\citep{FS25}) and numerical accounting.}
    \label{fig:PREAMBLE}
\end{figure}

We perform the privacy analysis for this setting using our PLD accounting methods, combining the results of Theorems \ref{thm:num_acc_RA} and \ref{thm:PLD_subsam} and Lemma \ref{lem:multAlloc}. 
Figure~\ref{fig:PREAMBLE} provides the results and compares them to the RDP-based bounds for all the settings in \citep{AFKRT26}. The sample size is $n = 6 \cdot 10^{5}$, the model dimension is $d = 2^{20}$, the communication constant is $C = 2^{15}$, and the number of epochs is $E=10$, where the expected batch size is $n q$. The y-axis represents the ratio between the noise $\sigma$ required to ensure $(\eps=1, \delta=10^{-6})$-DP using the PREAMBLE method and the standard private mean estimation via Gaussian noise addition (which corresponds to standard DP-SGD). As can be seen from the results, our accounting significantly improves on RDP-based bounds even in the relatively challenging setting where composition over numerous rounds amplifies approximation errors.

    \Arxiv{The large number of composition steps requires extremely tight bounds of the base random allocation PLD, resulting in a heavy computational load. The results were computed using the multiplicative-spacing method described in this work with $10^{6}$ bins for both add and remove directions.}

\para{Runtime.}
In Figure \ref{fig:runtime_experiment} we depict the runtime on a personal laptop as a function of the gap between upper and lower bounds\Arxiv{ (the lower bound is produced by the variant in Remark~\ref{rem:lower})} computed on the same parameter, as controlled by the tightness parameter $\alpha$ for several values of possible allocations $t$. As shown in Figure \ref{fig:gap_vs_disc}, this gap is proportional to the discretization $\alpha$, which is inversely proportional to the grid size. The results match the theoretical derivation in Theorem \ref{thm:num_acc_RA}, with runtime scaling $\sim \alpha^{-2}$. They also show that this method is practical and requires at most a few tens of seconds in most reasonable parameter regimes.

\begin{figure}[H]
    \centering
    \includegraphics[width=
    \ifconferencemode
\else
    0.8
    \fi
    \linewidth]{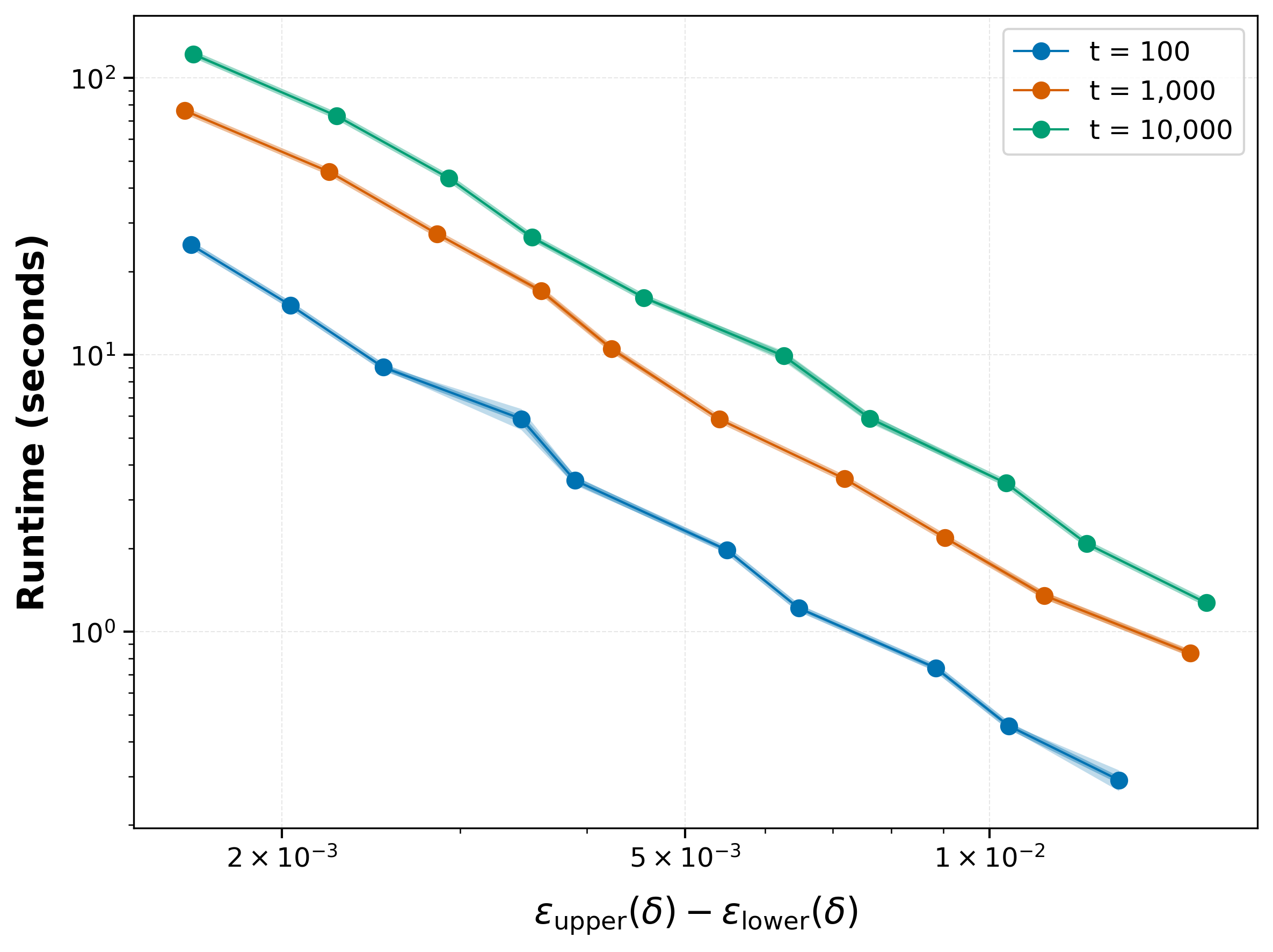}
    \caption{Runtime as a function of discretization $\alpha$ and number of steps $t$ on Apple MacBook Pro M1 averaged over $10$ runs (shaded area depicts the standard deviation).}
    \label{fig:runtime_experiment}
\end{figure}

\para{Privacy-utility trade-off.} \citet{CGHLKKMSZ24} demonstrated the utility advantage of the random allocation scheme relative to Poisson subsampling. In their experiments they train neural networks on a dataset of size $n$ using DP-SGD with a fixed noise level $\sigma$ and expected batch size $b$, and show that sampling using random allocation with $k=1, t=n/b$ results in better utility than Poisson sampling with $q=b/n$ (Figures \ref{fig:delta_comparison_partial} and \ref{fig:delta_comparison} correspond to the parameters of these experiments). This improvement is mainly attributed to the reduced variance in the number of times any single element participates in a training epoch. As mentioned above, our numerical bounds indicate that the privacy parameters of the random allocation scheme with the Gaussian mechanism are slightly better than those of Poisson for the same noise level and expected batch size for nearly all parameter regimes. The combination of these results directly implies that the utility-privacy tradeoff of the random allocation scheme is better than the corresponding one for Poisson in the explored settings.

\ifconferencemode
\else
    This phenomenon was additionally demonstrated in \citet[Appendix H]{FS25} using a toy example to compare the privacy-utility tradeoff of the Poisson and allocation schemes, where the mean of a dataset of size $n$ is estimated using an $n$-step scheme. In the one-dimensional setting a dataset $\dataset \in \{0,1\}^{n}$ is sampled i.i.d. from a Bernoulli distribution with expectation $p \in [0,1]$, at each iteration, the algorithm reports a noisy sum of the elements in the corresponding subset $\out_{i}$, and the estimated expectation $\hat{p} \coloneqq \frac{1}{n} \sum_{i \in [t]} \out_{i}$ is compared to the true value of $p$. In the multi-dimensional setting all but one coordinates are fixed to $0$, so changing the dimension only influences the scale of the added noise.

    Using their analytical bound they derived an upper bound on the scale of the noise required to achieve some fixed privacy level using the two schemes. They have shown the tradeoff between the accuracy degradation resulting from the small increase in noise scale required for the allocation scheme relative to Poisson, and the one induced by the additional sampling noise of the Poisson sampling. Using our new analysis, this tradeoff no longer holds, as the random allocation scheme requires slightly \emph{lower} level of noise, resulting in higher level of accuracy in all settings. Full experimental settings can be found in Appendix \ref{apd:expRes}.
\fi
\begin{movable}{mov:utility_comparison}{}
    \begin{figure}[H]
        \centering
        \includegraphics[width=1\linewidth]{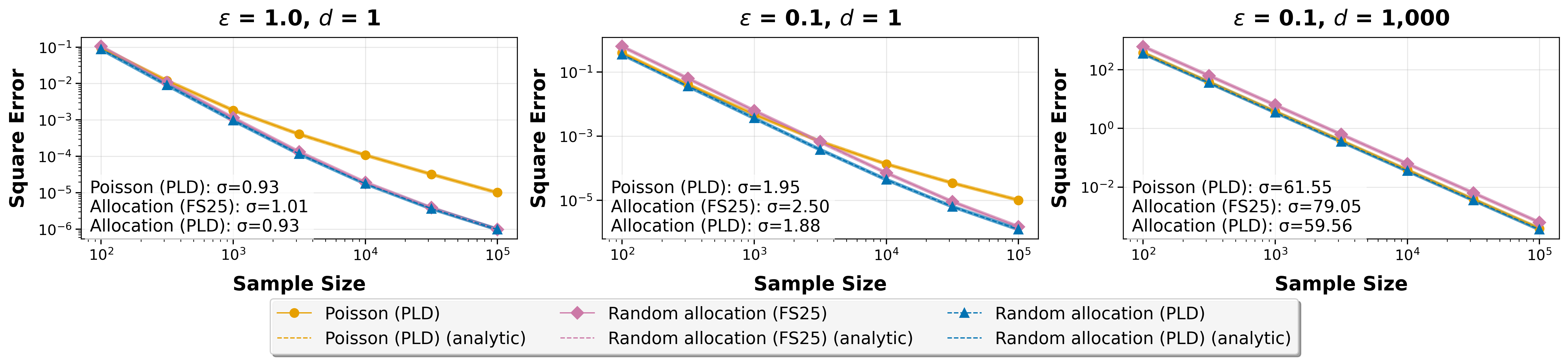}
        \caption{Analytical and empirical square error for the Poisson and random allocation scheme using both the combined method in \citet{FS25} and our PLD accounting, for various values of $\varepsilon$ and $d$ (which corresponds to an increase in sensitivity). We set $p = 0.9$, $t = 10^{3}$, $\delta = 10^{-10}$.}
        \label{fig:utility_comparison}
    \end{figure}
\end{movable}

In Appendix \ref{apd:expRes} we provide a number of additional numerical evaluations, including runtime/accuracy trade-offs, comparison with MCMC-based privacy bounds,  bounds for $k>1$ and comparison with an FFT-based implementation of convolutions for random allocation.

\section{Discussion}
Our work introduces the first efficient and tight numerical privacy accounting method for random allocation. The resulting privacy bounds, together with utility bounds for this sampling scheme, establish random allocation as a better and more practical alternative to Poisson subsampling, which had previously been the only sampling approach used for DP-SGD with valid privacy bounds.
We also formalize PLD realization as a natural representation of privacy loss bounds for numerical accounting, and demonstrate how subsampling and random allocation can be computed directly on this representation.

\section*{Acknowledgments}
We thank Matthew Regehr for his thoughtful comments.
Shenfeld's work was supported in part by the Apple Scholars in AI/ML PhD Fellowship, ERC grant 101125913, Simons Foundation Collaboration 733792, Israel Science Foundation (ISF) grant 2861/20, and a grant from the Israeli Council of Higher Education. Views and opinions expressed are, however, those of the author(s) only and do not necessarily reflect those of the European Union or the European Research Council Executive Agency. Neither the European Union nor the granting authority can be held responsible for them.
\ifconferencemode
\else
\newpage
\fi

\bibliography{bibliography}
\bibliographystyle{ICML_2026}

\newpage
\appendix
\onecolumn

\newpage

\section{Missing Details from Section \ref{sec:PLDreal}}\label{apd:PLDreal}
\showstored{mov:subsamPLDdom}

We start by stating some properties of PLD realizations.

\begin{claim}\label{clm:PLDisRealz}
    For any two distributions $P, Q$ over the same domain, $\lossRVfunc{P}{Q}$ is a PLD realization.
\end{claim}

\begin{proof}
    From the definition,
    \[
        \mathbb{E}\left[e^{-\lossRVfunc{P}{Q}}\right] = \underset{\omega \sim P}{\mathbb{E}}\left[e^{-\lossFunc{\omega}{P}{Q}}\right] = \underset{\omega \sim P}{\mathbb{E}}\left[\frac{Q(\omega)}{P(\omega)}\right] = \underset{\omega \sim Q}{\mathbb{E}}\left[\mathbbm{1}_{P(\omega) > 0} \right] \le 1,
    \]
    and $\lossFunc{\omega}{P}{Q} = -\infty \Rightarrow P(\omega) = 0$.
\end{proof}

\begin{claim}\label{clm:PLDdual}
    Given a PLD realization $\lossRV$, we have (1) for any $l \in \reals$, $\lossFunc{l}{f_{\lossRV}}{f_{-\dual{\lossRV}}} = l$; (2) the PLD between $f_{\lossRV}$ and $f_{-\dual{\lossRV}}$ is $f_{\lossRV}$; and (3) $f_{\dual{\lossRV}}$ is a PLD realization as well.
\end{claim}

\begin{proof}
    From the definition of PLD realization $f_{\lossRVfunc{f_{\lossRV}}{f_{-\dual{\lossRV}}}}(-\infty) = f_{\lossRV}(-\infty) = 0$, and from the definition of the PLD dual $f_{-\dual{\lossRV}}(\infty) = 0$ so $f_{\lossRVfunc{f_{\lossRV}}{f_{-\dual{\lossRV}}}}(\infty) = f_{\lossRV}(\infty)$.
    
    Combining the definition of the privacy loss and the PLD dual we have,
    \[
        \lossFunc{l}{\lossRV}{-\dual{\lossRV}} = \ln\left(\frac{f_{\lossRV}(l)}{f_{-\dual{\lossRV}}(l)}\right) = \ln\left(\frac{f_{\lossRV}(l)}{f_{\lossRV}(l) e^{-l}}\right) = \ln\left(e^{l}\right) = l.
    \]
    
    The second part is a direct result of the fact,
    \[
        f_{\lossRVfunc{f_{\lossRV}}{f_{-\dual{\lossRV}}}}(l) = \int_{\reals} \delta(\lossFunc{x}{f_{\lossRV}}{f_{-\dual{\lossRV}}} - l) f_{\lossRV}(x) d x = \int_{\reals} \delta(x - l) f_{\lossRV}(x) d x = f_{\lossRV}(l).
    \]

    Similarly, for the third part,
    \[
        \mathbb{E}\left[e^{-\dual{\lossRV}}\right] = \int_{\reals} e^{-l} f_{\dual{\lossRV}}(l) dl = \int_{\reals} f_{\lossRV}(-l) dl \le 1,
    \]
    and by definition $f_{\dual{\lossRV}}(-\infty) = 0$.
\end{proof}

\begin{proof}[Proof of Theorem \ref{thm:PLD_subsam}]
    We first provide an explicit relation of the privacy losses.
    \[
        \lossFunc{\omega}{P_{\lambda}}{Q} = \ln\left(\frac{P_{\lambda}(\omega)}{Q(\omega)}\right) = \ln\left(\frac{\lambda P(\omega) + (1-\lambda) Q(\omega)}{Q(\omega)}\right) = \ln\left(1 + \lambda \left(e^{\lossFunc{\omega}{P}{Q}} - 1\right)\right)
    \]
    Using this identity we get,
    \begin{align*}
        F_{\lossRVfunc{P_{\lambda}}{Q}}(l) & = \prob{\omega \sim P_{\lambda}}{\lossFunc{\omega}{P_{\lambda}}{Q} \le l}
        \\ & = \prob{\omega \sim P_{\lambda}}{\ln\left(1 + \lambda \left(e^{\lossFunc{\omega}{P}{Q}} - 1\right)\right) \le l}
        \\ & = \prob{\omega \sim P_{\lambda}}{\lossFunc{\omega}{P}{Q} \le \ln\left(1 + \left(e^{l} - 1\right) / \lambda\right)}
        \\ & = \prob{\omega \sim P_{\lambda}}{\lossFunc{\omega}{P}{Q} \le \phi_{\lambda}(l)}
        \\ & = \lambda \prob{\omega \sim P}{\lossFunc{\omega}{P}{Q} \le \phi_{\lambda}(l)} + (1-\lambda) \prob{\omega \sim Q}{\lossFunc{\omega}{P}{Q} \le \phi_{\lambda}(l)}
        \\ & = \lambda \prob{\omega \sim P}{\lossFunc{\omega}{P}{Q} \le \phi_{\lambda}(l)} + (1-\lambda) \prob{\omega \sim Q}{-\lossFunc{\omega}{Q}{P} \le \phi_{\lambda}(l)}
        \\ & = \lambda F_{\lossRVfunc{P}{Q}}(\phi_{\lambda}(l)) + (1-\lambda) F_{-\lossRVfunc{Q}{P}}(\phi_{\lambda}(l)),
    \end{align*}
    which implies,
    \[
        f_{\lossRVfunc{P_{\lambda}}{Q}}(l) = \lambda f_{\lossRVfunc{P}{Q}}(\phi_{\lambda}(l)) + (1-\lambda) f_{-\lossRVfunc{Q}{P}}(\phi_{\lambda}(l)).
    \]
    We note that this is an identity between probability \emph{masses}: $f$ denotes the law of the realization (its PMF in the discrete case, or the pushforward of its distribution in general), and $\phi_{\lambda}$ merely relabels the loss axis. A monotone relabeling relocates mass between values but neither creates nor destroys it, so the CDF\textemdash{}and hence $f$ read as mass\textemdash{}composes without a Jacobian. The $\phi_{\lambda}'$ factor would appear only if $f$ were interpreted as a Lebesgue density (mass per unit loss), a representation we do not use.
    
    Similarly,
    \begin{align*}
        F_{\lossRVfunc{Q}{P_{\lambda}}}(l) & = \prob{\omega \sim Q}{\lossFunc{\omega}{Q}{P_{\lambda}} \le l}
        \\ & = \prob{\omega \sim Q}{-\ln\left(1 + \lambda \left(e^{\lossFunc{\omega}{P}{Q}} - 1\right)\right) \le l}
        \\ & = \prob{\omega \sim Q}{\lossFunc{\omega}{P}{Q} \ge \ln\left(1 + \left(e^{-l} - 1\right) / \lambda\right)}
        \\ & = \prob{\omega \sim Q}{\lossFunc{\omega}{Q}{P} \le -\ln\left(1 + \left(e^{-l} - 1\right) / \lambda\right)}
        \\ & = \prob{\omega \sim Q}{\lossFunc{\omega}{Q}{P} \le -\phi_{\lambda}(-l)}
        \\ & = F_{\lossRVfunc{Q}{P}}(-\phi_{\lambda}(-l)),
    \end{align*}
    which implies, $f_{\lossRVfunc{Q}{P_{\lambda}}}(l) = f_{\lossRVfunc{Q}{P}}(-\phi_{\lambda}(-l))$.
\end{proof}

\section{Missing Details from Section \ref{sec:PLDconv}}\label{apd:PLDconv}

\begin{claim}\label{clm:stochDomPLDreal}
	If a random variable $\upBnd$ stochastically dominates a PLD realization $\lossRV$ then it is a PLD realization as well.
\end{claim}

\begin{proof}
    From the definition of stochastic domination $f_{\upBnd}(-\infty) = 0$, and from the monotonicity of the function $e^{-x}$ we have that $e^{-\upBnd}$ is dominated by $e^{-\lossRV}$ so $\expect{}{e^{-\upBnd}} \le \expect{}{e^{-\lossRV}} \le 1$ (in the probability literature this is typically referred to as ``stochastically smaller than'').
\end{proof}

\showstored{mov:stochDomImpDom}

\begin{proof}\Arxiv{[Proof of Claim \ref{clm:stochDomImpDom}]}
    Using the fact that 
    \begin{align*}
        \HockeyStickFunc{e^{\eps}}{X} & = \expect{}{\left[1 - e^{\eps-X} \right]_{+}}
        \\ & = \int_{0}^{1} \prob{}{\left[1 - e^{\eps-X} \right]_{+} > t} dt
        \\ & = \int_{0}^{1} \prob{}{X > \eps - \ln(1 - t)} dt
        \\ & = \int_{0}^{1} \bar{F}_{X}\left(\eps - \ln(1 - t)\right) dt,
    \end{align*}

    we have,
    \[
        \HockeyStickFunc{e^{\eps}}{\lowBnd} = \int_{0}^{1} \bar{F}_{\lowBnd}\left(\eps - \ln(1 - t) \right) dt \le \int_{0}^{1} \bar{F}_{\upBnd}\left(\eps - \alpha - \ln(1 - t)\right) + \beta dt = \HockeyStickFunc{e^{\eps - \alpha}}{\upBnd} + \beta.
    \]
\end{proof}

\begin{claim}\label{clm:PLDdualDom}
    Stochastic domination is not maintained under some of the transformations we consider even for PLD realizations, but its approximate form does.
    \begin{enumerate}
        \item There exist two PLD realizations $\lowBnd, \upBnd$ such that $\upBnd$ stochastically dominates $\lowBnd$ but $\dual{\upBnd}$ neither stochastically dominates $\dual{\lowBnd}$ nor is stochastically dominated by it.
        \item There exist $\lambda \in (0,1]$ and two random variables $\lowBnd, \upBnd$ such that $\upBnd$ stochastically dominates $\lowBnd$ but $\rem{\varphi}_{\lambda}(\upBnd)$ neither stochastically dominates $\rem{\varphi}_{\lambda}(\lowBnd)$ nor is stochastically dominated by it.
        \item For any two PLD realizations $\lowBnd, \upBnd$, if $\lowBnd{} \preceq \upBnd$ then $\dual{\upBnd} \preceq_{(0, c)} \dual{\lowBnd}$ where $c \coloneqq f_{\dual{\upBnd}}(\infty) - f_{\dual{\lowBnd}}(\infty) = \mathbb{E}\left[e^{-\lowBnd}\right] - \mathbb{E}\left[e^{-\upBnd}\right]$.
    \end{enumerate}
\end{claim}

\begin{proof}
    \textbf{First part:} Consider the random variables $\lowBnd, \upBnd$ over $\{\ln(2), \ln(4)\}$, such that $f_{\lowBnd}(\ln(2)) = f_{\upBnd}(\ln(4)) = 0.6$ and $f_{\lowBnd}(\ln(4)) = f_{\upBnd}(\ln(2)) = 0.4$. Clearly, both are PLD realizations, and $\upBnd$ stochastically dominates $\lowBnd$. Their duals are $f_{\dual{\lowBnd}}(-\ln(4)) = 0.1, f_{\dual{\lowBnd}}(-\ln(2)) = 0.3, f_{\dual{\lowBnd}}(\infty) = 0.6$, and $f_{\dual{\upBnd}}(-\ln(4)) = 0.15, f_{\dual{\upBnd}}(-\ln(2)) = 0.2, f_{\dual{\upBnd}}(\infty) = 0.65$ which do not stochastically dominate each other.

    \textbf{Second part:} Consider the random variables $\lowBnd, \upBnd$, such that $f_{\lowBnd}(0) = f_{\upBnd}(1) = 1$. Clearly, both are PLD realizations, and $\upBnd$ stochastically dominates $\lowBnd$. Their duals are $f_{\dual{\lowBnd}}(0) = 1$, and $f_{\dual{\upBnd}}(-1) = 1/e, f_{\dual{\upBnd}}(\infty) = 1-1/e$. Setting $\lambda = 0.5$ we have $f_{\rem{\varphi}_{0.5}(\lowBnd)}(0) = 1$ and $f_{\rem{\varphi}_{0.5}(\upBnd)}(\ln((e+1)/2)) = (e+1)/2e, f_{\rem{\varphi}_{0.5}(\upBnd)}(\ln(1/2)) = (e-1)/2e$ which do not stochastically dominate each other.

    \textbf{Third part:} Denoting $g_{l}(x) \coloneqq e^{-x} \cdot \mathbbm{1}_{\{x \le -l\}}$ we have for $W \in \{\lowBnd, \upBnd\}$,

    \[
        F_{\dual{W}}(l) = \int_{-\infty}^{l} f_{\dual{W}}(t) dt = \int_{-l}^{\infty} e^{-t} f_{W}(t) dt = \expect{}{e^{-W}} - \expect{}{g_{l}(W)} = 1 - f_{\dual{W}}(\infty) - \expect{}{g_{l}(W)}.
    \] 

    Using the fact that for any two random variables $\lowBnd, \upBnd$ if $\lowBnd \preceq \upBnd$ and $\varphi$ is a non-increasing function then $\mathbb{E}[\varphi(\upBnd)] \le \mathbb{E}[\varphi(\lowBnd)]$, assuming both expectations exist (See e.g., \cite{SS07} Eq. 1.A.7), we get
    \[
        \bar{F}_{\dual{\upBnd}}(l) = \expect{}{g_{l}(\upBnd)} + f_{\dual{\upBnd}}(\infty) \le \expect{}{g_{l}(\lowBnd)} + f_{\dual{\upBnd}}(\infty) = \bar{F}_{\dual{\lowBnd}}(l) + f_{\dual{\upBnd}}(\infty) - f_{\dual{\lowBnd}}(\infty).
    \]
\end{proof}

\showstored{mov:stochDomComp}

\begin{proof}\Arxiv{[Proof of Claim \ref{clm:stochDomComp}]}
    From the definition, for any $l \in (-\infty, \infty)$,
    \begin{align*}
        \bar{F}_{\lowBnd_{1}+\lowBnd_{2}}(l) & = \int_{-\infty}^{\infty} f_{\lowBnd_{1}}(x) \bar{F}_{\lowBnd_{2}}(l-x) dx
        \\ & \le \int_{-\infty}^{\infty} f_{\lowBnd_{1}}(x) (\bar{F}_{\upBnd_{2}}(l-x-\alpha_{2}) + \beta_{2}) dx
        \\ & = \beta_{2} + \bar{F}_{\lowBnd_{1}+\upBnd_{2}}(l-\alpha_{2})
        \\ & = \beta_{2} + \int_{-\infty}^{\infty} f_{\upBnd_{2}}(x) \bar{F}_{\lowBnd_{1}}(l-\alpha_{2}-x) dx
        \\ & \le \beta_{2} + \int_{-\infty}^{\infty} f_{\upBnd_{2}}(x) (\bar{F}_{\upBnd_{1}}(l-\alpha_{2}-x-\alpha_{1}) + \beta_{1}) dx
        \\ & = \bar{F}_{\upBnd_{1} + \upBnd_{2}}(l-(\alpha_{1}+\alpha_{2})) + (\beta_{1} + \beta_{2}).
    \end{align*}

    It remains to verify the two endpoints. For $l = -\infty$, note that $\lowBnd_{1} + \lowBnd_{2} > -\infty$ if and only if both $\lowBnd_{1} > -\infty$ and $\lowBnd_{2} > -\infty$ (and likewise for $\upBnd_{1}, \upBnd_{2}$), so by independence
    \begin{align*}
        \bar{F}_{\lowBnd_{1}+\lowBnd_{2}}(-\infty) & = \bar{F}_{\lowBnd_{1}}(-\infty) \cdot \bar{F}_{\lowBnd_{2}}(-\infty)
        \\ & \le \min\{\bar{F}_{\upBnd_{1}}(-\infty) + \beta_{1}, 1\} \cdot \min\{\bar{F}_{\upBnd_{2}}(-\infty) + \beta_{2}, 1\}
        \\ & \le \bar{F}_{\upBnd_{1}}(-\infty) \cdot \bar{F}_{\upBnd_{2}}(-\infty) + \beta_{1} + \beta_{2}
        \\ & = \bar{F}_{\upBnd_{1}+\upBnd_{2}}(-\infty) + \beta_{1} + \beta_{2},
    \end{align*}
    where the second inequality uses $\bar{F}_{\lowBnd_{i}}(-\infty) \le \bar{F}_{\upBnd_{i}}(-\infty) + \beta_{i}$ together with $\bar{F}_{\lowBnd_{i}}(-\infty) \le 1$. 
    
    For $l = \infty$ the bound is immediate, since $\bar{F}_{\lowBnd_{1} + \lowBnd_{2}}(\infty) = \bar{F}_{\upBnd_{1} + \upBnd_{2}}(\infty) = 0$.
\end{proof}

\begin{proof}[Proof of Theorem \ref{thm:PLDrandAlloc}]
    From the definition,
    \begin{align*}
        \lossFunc{\bar{\omega}_{1:t}}{\bar{P}_{t}}{Q^{t}} & = \ln\left(\frac{\bar{P}_{t}(\bar{\omega}_{1:t})}{Q^{t}(\bar{\omega}_{1:t})}\right)
        \\ & = \ln\left(\frac{\frac{1}{t} \sum_{i \in [t]} Q^{i-1}(\bar{\omega}_{1:i-1}) \times P(\omega_{i}) \times Q^{t-i}(\bar{\omega}_{i+1:t})}{Q^{t}(\bar{\omega}_{1:t})}\right)
        \\ & = \ln\left(\frac{1}{t} \sum_{i \in [t]} \frac{P(\omega_{i})}{Q(\omega_{i})}\right)
        \\ & = \ln\left(\frac{1}{t} \sum_{i\in [t]} e^{\lossFunc{\omega_{i}}{P}{Q}}\right).
    \end{align*}
    
    From symmetry, the random variable defined by $\lossFunc{\bar{\omega}_{1:t}}{\bar{P}_{t}}{Q^{t}}$ where $\bar{\omega}_{1:t}$ is sampled from $Q^{i-1} \times P \times Q^{t-i}$ is identically distributed for any $i \in [t]$, so WLOG we choose $\bar{\omega}_{1:t} \sim P \times Q^{t-1}$ and get 
    \[
        \lossFunc{\bar{\omega}_{1:t}}{\bar{P}_{t}}{Q^{t}} = \ln\left(\frac{1}{t} e^{\lossFunc{\omega_{1}}{P}{Q}} + \frac{1}{t} \sum_{i\in [2:t]} e^{-\lossFunc{\omega_{i}}{Q}{P}}\right).
    \]
    Combining this with the fact $\dual{\lossRVfunc{P}{Q}} = \lossRVfunc{Q}{P}$ (Claim \ref{clm:PLDdual}) this implies, 
    \[
        \lossRVfunc{\bar{P}_{t}}{Q^{t}} = \ln\left(\frac{1}{t} \left(e^{\lossRVfunc{P}{Q}} + \sum_{i \in [t-1]} e^{-\lossRVfunc{Q}{P}^{(i)}}\right)\right) = \ln\left(\frac{1}{t} \left(e^{\lossRVfunc{P}{Q}} + \sum_{i \in [t-1]} e^{-\dual{\lossRVfunc{P}{Q}}^{(i)}}\right)\right),
    \]
    where the superscript $(i)$ denotes independent copies.

    Similarly, sampling all elements from $Q^{t}$ we have,
    \[
        \lossFunc{\bar{\omega}_{1:t}}{Q^{t}}{\bar{P}_{t}} = -\lossFunc{\bar{\omega}_{1:t}}{\bar{P}_{t}}{Q^{t}} = -\ln\left(\frac{1}{t} \sum_{i\in [t]} e^{-\lossFunc{\omega_{i}}{Q}{P}}\right)
    \]
    which implies $\lossRVfunc{Q^{t}}{\bar{P}_{t}} = -\ln\left(\frac{1}{t} \sum_{i \in [t]} e^{-\lossRVfunc{Q}{P}^{(i)}}\right)$.

\end{proof}

\showstored{mov:mult_tight}

\begin{proof}\Arxiv{[Proof of Lemma \ref{lem:mult_tight}]}
    We first notice that,
    \[
        \bar{F}_{e^{\lowBnd_{i}}}\left(e^{x}\right) = \bar{F}_{\lowBnd_{i}}\left(x\right) \le \bar{F}_{\upBnd_{i}}\left(x-\alpha\right) + \beta_{i} = \bar{F}_{\upBnd_{i}+\alpha}\left(x\right) + \beta_{i} = \bar{F}_{e^{\upBnd_{i} + \alpha}}\left(e^{x} \right) + \beta_{i}.
    \]
    Using this bound we get,
        \begin{align*}
        \bar{F}_{\ln\left(e^{\lowBnd_{1}}+e^{\lowBnd_{2}} \right)}(x) & = \bar{F}_{e^{\lowBnd_{1}}+e^{\lowBnd_{2}}}\left(e^{x}\right)
        \\ & = \int_{-\infty}^{\infty} f_{e^{\lowBnd_{1}}}(y) \bar{F}_{e^{\lowBnd_{2}}}\left(e^{x}-y\right) dy
        \\ & \le \int_{-\infty}^{\infty} f_{e^{\lowBnd_{1}}}(y) \left(\bar{F}_{e^{\upBnd_{2}+\alpha}}\left(e^{x}-y\right) + \beta_{2} \right) dy
        \\ & = \bar{F}_{e^{\lowBnd_{1}}+e^{\upBnd_{2}+\alpha}}\left(e^{x}\right) + \beta_{2}
        \\ & = \bar{F}_{e^{\lowBnd_{1}-\alpha}+e^{\upBnd_{2}}}\left(e^{x-\alpha}\right) + \beta_{2}
        \\ & = \beta_{2} + \int_{-\infty}^{\infty} f_{e^{\upBnd_{2}}}(y) \bar{F}_{e^{\lowBnd_{1}-\alpha}}\left(e^{x-\alpha}-y\right) dy
        \\ & \le \beta_{2} + \int_{-\infty}^{\infty} f_{e^{\upBnd_{2}}}(y) \left(\bar{F}_{e^{\upBnd_{1}}}\left(e^{x-\alpha}-y\right) + \beta_{1} \right) dy
        \\ & = \bar{F}_{e^{\upBnd_{1}}+e^{\upBnd_{2}}}\left(e^{x-\alpha}\right) + (\beta_{1} + \beta_{2}) 
        \\ & = \bar{F}_{\ln\left(e^{\upBnd_{1}}+e^{\upBnd_{2}} \right)}\left(x-\alpha\right) + (\beta_{1} + \beta_{2}),
    \end{align*}
    and from Claim \ref{clm:stochDomComp},
    \[
        \bar{F}_{\ln\left(e^{\lowBnd_{1}}+e^{\lowBnd_{2}} \right)}(\infty) = \bar{F}_{e^{\lowBnd_{1}} + e^{\lowBnd_{2}}}(\infty) \le \bar{F}_{e^{\upBnd_{1}} + e^{\upBnd_{2}}}(\infty) + \beta_{1} + \beta_{2} = \bar{F}_{\ln\left(e^{\upBnd_{1}}+e^{\upBnd_{2}} \right)}(\infty) + \beta_{1} + \beta_{2}.
    \]
\end{proof}

\begin{proof}[Proof of Theorem \ref{thm:num_acc_RA}]
    We start by pointing out two simple facts:
    \begin{enumerate}
        \item The number of times Algorithm \ref{alg:conv} is called with input $t$ equals the number of pairwise convolutions performed by exponentiation by squaring in Algorithm \ref{alg:self_conv}, namely $\lfloor \log_{2}(t) \rfloor + \operatorname{popcount}(t) - 1 \le 2 \lceil \log_{2}(t) \rceil$, where $\operatorname{popcount}(t)$ is the number of nonzero bits in the binary representation of $t$.
        \item If both input distributions to Algorithm \ref{alg:conv} are defined over the grids $\left[a \cdot e^{\alpha \cdot i}\right]_{i\in[n+1]}$, $\left[b \cdot e^{\alpha \cdot i}\right]_{i\in[n+1]}$, then the output distribution is defined over the grid $\left[(a+b) \cdot e^{\alpha \cdot i}\right]_{i\in[n+1]}$. This is because by definition, the output grid is supported over the range $((a+b) \cdot e^{\alpha}, (a+b) \cdot e^{(n+1) \alpha})$, and it is spaced with constant ratio $e^{\alpha}$, so its size is $\log_{e^{\alpha}}\left(\frac{(a+b) \cdot e^{(n+1) \alpha}}{(a+b) \cdot e^{\alpha}}\right) = \ln\left(e^{\alpha n}\right)/\alpha = n$.
    \end{enumerate}

    We state the analysis first in terms of the remove direction, then point out the slight changes for the add direction.
    
    \textbf{Validity:} In the remove direction the algorithm computes the exact dual $\dual{\rem{\lossRV}}$ of the input PLD realization; hence, from Theorem \ref{thm:PLDrandAlloc}, if all discretization and convolution steps were lossless, then $\rem{\lossRV}_{t} = \rem{\psi}_{t}(\rem{\lossRV})$ and $\add{\lossRV}_{t} = \add{\psi}_{t}(\add{\lossRV})$. From Claim \ref{clm:stochDomImpDom} it suffices to implement all re-discretization steps in a stochastically dominating manner to complete the proof.
    The first step consists of the construction of a random variable dominating the input PLD realization (and, in the remove direction, one dominating its dual), and at each step of the convolution, the output random variable dominates the exact convolution of the two input random variables. Since the exponent and logarithm are monotonically increasing functions, they preserve domination. The proof of the add direction is identical, except that we compute dominated random variables of each step, since the last part is applying the monotonically decreasing transformation $-\ln(x)$.

    \textbf{Tightness:} We note that there are two points in the process where the domination is not tight. The first is in the construction of the random variables dominating the input PLD realization and its dual, and the second is the re-discretization at each convolution step. From the first fact, there are $\le 2 \lceil \log_{2}(t) \rceil$ convolution steps which, together with the initial discretization, are the only stages where accuracy degrades. From Lemma \ref{lem:mult_tight}, the effect of each discretization step does not increase with subsequent convolutions, so from Claim \ref{clm:stochDomImpDom} the overall discretization error is their sum. The effect of the tail truncation does grow with the number of convolutions, but it takes place only in the initial discretization. From the definition of $\alpha', \beta'$, the overall relative error is $(\alpha, \beta)$.
    
    \textbf{Computational complexity}: Since we directly compute the convolution, each one requires $O(n\cdot m)$ steps where $n, m$ are the number of bins in the two discrete random variables. From the second fact, the grid size remains identical to the initial size $n_{0}$. Combining this with the first fact implies the computational complexity of this algorithm is $O\left(n_{0}^{2} \cdot \log_{2}(t) \right)$. Using the fact that the bins were evenly spaced in the PLD between the $\beta'$ and $1-\beta'$ quantiles with width $\alpha'$ and the definition of $\alpha', \beta'$ we get,
    \[
        n_{0} = \frac{\mathtt{IQR}_{\beta'}}{\alpha'} = \frac{\mathtt{IQR}_{\beta/t}}{\alpha} \cdot \left(2 \left\lceil \log_{2}(t) \right \rceil + 1 \right),
    \]
    which completes the proof.
\end{proof}

The next claim shows how two bounds can be combined to form a tighter one. A fact that can be used to provide tighter bounds by combining two convolution techniques, as discussed in the first part of Section \ref{sec:numRes}.
\begin{claim}\label{clm:bndComb}
    Given three random variables $\lowBnd_{1}, \lowBnd_{2}, X$, if $\lowBnd_{1} \preceq X$ and $\lowBnd_{2} \preceq X$ then $\lowBnd_{1} \preceq \lowBnd \preceq X$ and $\lowBnd_{2} \preceq \lowBnd \preceq X$, where $\lowBnd$ is defined via its CCDF, $\forall l \in [-\infty, \infty]; \bar{F}_{\lowBnd}(l) \coloneqq \max\{\bar{F}_{\lowBnd_{1}}(l), \bar{F}_{\lowBnd_{2}}(l)\}$. Similarly, if $X \preceq \upBnd_{1}$ and $X \preceq \upBnd_{2}$ then $X \preceq \upBnd \preceq \upBnd_{1}$ and $X \preceq \upBnd \preceq \upBnd_{2}$, where $\forall l \in [-\infty, \infty]; \bar{F}_{\upBnd}(l) \coloneqq \min\{\bar{F}_{\upBnd_{1}}(l), \bar{F}_{\upBnd_{2}}(l)\}$.
\end{claim}

\newpage
\section{Full Implementation Details}\label{apd:impDtls}
In this section we provide detailed description of the algorithm implementation. We represent discrete random variables $X$ as a pair $(\bar{l}, f_{X})$, where $\bar{l}$ is an array of increasing values and $f_{X}$ is their corresponding probabilities. We use $0$-based indexing notation and denote the last cell as $l[-1]$. We note that if $X$ is a PLD realization $l[0] = -\infty$, $l[-1] = \infty$, and $f_{X}[0] = 0$. In the context of $\exp$-PLDs we have $e^{l[0]} = e^{-\infty} = 0$ and $e^{l[-1]} = e^{\infty} = \infty$. This affects the convolution since for any scalar $x\pm\infty = \pm\infty$ but $x+0=x$. Throughout we use notation such as $\exp(\lossRV)$ and $\ln(\lossRV)$ to denote the operation over the loss values that keeps the PDF unchanged, and in the case of monotonically decreasing operation\textemdash{}reverses the order of the losses and PDF.

We start by describing the main functions, with remove (\ref{alg:main_alg_rem}) and add (\ref{alg:main_alg_add}) variants as separate algorithms, followed by their building blocks. Our implementation uses the exact number of pairwise convolutions in exponentiation by squaring, $\lfloor \log_{2}(t)\rfloor + \operatorname{popcount}(t)-1$, where $\operatorname{popcount}(t)$ is the number of nonzero bits in the binary representation of $t$. In the pseudocode and analysis we use the simpler upper bound $2 \left\lceil \log_{2}(t) \right\rceil$.

~

\begin{figure}[htbp]
    \begin{minipage}{0.48\textwidth}
        \begin{algorithm}[H]
            \caption{Random allocation numerical accounting (remove) \textbf{rand-alloc-rem($\rem{\lossRV}, t, \alpha, \beta$)}}\label{alg:main_alg_rem}
            \begin{algorithmic}
               \REQUIRE $\rem{\lossRV}$-PLD realization,~ $t$-number of allocations,~ $(\alpha, \beta)$-domination tightness parameters
                \ENSURE$\rem{L}_{t}$-PLD realization dominating $\rem{\psi}_{t}\left(\rem{\lossRV} \right)$
                
                \STATE $\alpha' \gets \frac{\alpha}{2 \left \lceil \log_{2}(t) \right\rceil + 1}$, $\beta' \gets \frac{\beta}{t}$
        
                \STATE $L \gets \text{disc-dist}(\rem{\lossRV}, \alpha', \beta')$

                \IF{$t=1$}
                    \RETURN $\rem{L}_{t} = L$
                \ENDIF
        
                \STATE $D \gets \text{PLD-dual}(\rem{\lossRV})$
                            
                \STATE $D' \gets \text{disc-dist}(-D, \alpha', \beta')$

                \STATE $eL_{1} \gets \exp(L)$, $eD_{1} \gets \exp(D')$
                            
                \STATE $eD_{t-1} \gets \text{self-conv}\left(eD_{1}, t-1, \text{upper}\right)$
                
                \STATE $eL_{t} \gets \text{conv}\left(eD_{t-1}, eL_{1}, \text{upper} \right)$
                                 
                \RETURN $\rem{L}_{t} = \ln(eL_{t}/t)$
            \end{algorithmic}
        \end{algorithm}
    \end{minipage}
    \hfill
    \begin{minipage}{0.48\textwidth}
        \begin{algorithm}[H]
            \caption{Random allocation numerical accounting (add) \textbf{rand-alloc-add($\add{\lossRV}, t, \alpha, \beta$)}}\label{alg:main_alg_add}
            \begin{algorithmic}
            
               \REQUIRE$\add{\lossRV}$-PLD realization,~ $t$-number of allocations,~ $(\alpha, \beta)$-domination tightness parameters
                \ENSURE$\add{L}_{t}$-PLD realization dominating $\add{\psi}_{t}\left(\add{\lossRV} \right)$
            
                \STATE $\alpha' \gets \frac{\alpha}{2 \left \lceil \log_{2}(t) \right\rceil + 1}$, $\beta' \gets \frac{\beta}{t}$

                \STATE $L \gets \text{disc-dist}(\add{\lossRV}, \alpha', \beta')$

                \STATE $eL_{1} \gets \exp(-L)$

                \STATE $eL_{t} \gets \text{self-conv}\left(eL_{1}, t, \text{lower}\right)$
                                
                \RETURN $\add{L}_{t} = -\ln(eL_{t}/t)$
            \end{algorithmic}
        \end{algorithm}
    \end{minipage}
\end{figure}

\begin{remark}
    A lower-bound variant of the algorithm produces tight numerical lower bounds on the PLD, obtained by rounding down instead of up in disc-dist and reversing the direction of conv and self-conv, and all results of Theorem \ref{thm:num_acc_RA} extend to this direction. Since the dual is extracted from the input realization before any discretization (rather than from its discretization), this direction is well defined even when down-rounding the loss would otherwise yield a random variable that is not a PLD realization and has no dual.
\end{remark}

The main component of these two algorithms is convolution, which is computed in Algorithm \ref{alg:self_conv} recursively using exponentiation by squaring. Given a number of steps $t \in \naturals$ we decompose it into a sum of powers of $2$, so we can compute the $t$-times self convolution as a convolution of the $2^{i}$-times self convolutions using Algorithm \ref{alg:conv}. For example, for $t = 13$ we can compute the $2, 4$ and $8$ convolutions, then convolve the random variable with its $4$th and $8$th convolution, for a total of $5$ convolutions.

Each convolution is computed in Algorithm \ref{alg:conv} by direct multiplication of the two input random variables, and discretized into a new grid computed by Algorithm \ref{alg:range_renorm}. It assumes both input random variables are defined over the same number of points, and that if $l[0] = -\infty$ and $l[-1] = \infty$ then $f_{X}[0] \cdot f_{X'}[-1] = f_{X}[-1] \cdot f_{X'}[0] = 0$ since $\infty - \infty$ is not defined. Since we convolve $\exp$-PLDs whose minimal value corresponds to $0$, this is never a problem. First it computes the output grid using Algorithm \ref{alg:range_renorm}, then it computes the PMF of $f_{\text{conv}}$ as the sum of $f_{X}[j] \cdot f_{X'}[k]$ over all pairs $j, k$ such that $\bar{l}[j] + \bar{l}'[k]$ belongs to the relevant bin, with the probability mass assigned to its left or right depending on the domination direction.

\begin{algorithm}[H]
    \caption{Self convolution \textbf{self-conv}($X, t$, dir)}\label{alg:self_conv}

    \begin{algorithmic}
       \REQUIRE$X$-random variable,~ $t$ number of convolutions, dir-domination direction (upper/lower)
        \ENSURE$X_{\text{acc}}$-convoluted random variable
    
        \STATE $X_{\text{base}} \gets X$
        
        \STATE $\text{init} \gets \text{False}$ 
        
        \WHILE{$t > 0$}
            \IF{$t$ is odd}
                \IF{$\text{init}$}
                    \STATE $X_{\text{acc}} \gets \text{conv}(X_{\text{base}}, X_{\text{acc}}, \text{dir})$
                \ELSE
                    \STATE $X_{\text{acc}} \gets X_{\text{base}}$
                    \STATE $\text{init} \gets \text{True}$ 
                \ENDIF
            \ENDIF

            \STATE $t \gets \lfloor t/2 \rfloor$
            \IF{$t > 0$}
                \STATE $X_{\text{base}} \gets \text{conv}(X_{\text{base}}, X_{\text{base}}, \text{dir})$
            \ENDIF
            \ENDWHILE
        \IF{$\text{init}$}
            \RETURN $X_{\text{acc}}$
        \ELSE
            \RETURN $X_{\text{base}}$
        \ENDIF
    \end{algorithmic}
\end{algorithm}

~

\begin{algorithm}[H]
    \caption{Convolve two random variables \textbf{conv}($X, X'$, dir)}\label{alg:conv}

    \begin{algorithmic}
       \REQUIRE$X, X'$-two random variables, dir-domination direction (upper/lower)
        \ENSURE$X_{\text{conv}}$-convoluted random variable
    
        \STATE $\bar{l}_{\text{conv}} \gets \text{range-renorm}(X, X')$
    
        \STATE $\bar{l} \gets$ values of $X$, $\bar{l}' \gets$ values of $X'$, $n \gets (\text{size of } \bar{l}) - 2$

        \IF{dir = 'upper'}
            \STATE $f_{\text{conv}}[0] \gets \underset{\bar{l}[j] + \bar{l}'[k] = \bar{l}_{\text{conv}}[0]}{\sum} f_{X}[j] \cdot f_{X'}[k]$, 
            
            \STATE $f_{\text{conv}}[1:n] \gets \left[\underset{\bar{l}[j] + \bar{l}'[k] \in (\bar{l}_{\text{conv}}[i-1], \bar{l}_{\text{conv}}[i]]}{\sum} f_{X}[j] \cdot f_{X'}[k] \right]_{i \in [n]}$
    
            \STATE $f_{\text{conv}}[n+1] \gets \underset{\bar{l}[j] + \bar{l}'[k] > \bar{l}_{\text{conv}}[n]}{\sum} f_{X}[j] \cdot f_{X'}[k]$            
        \ELSE
            \STATE $f_{\text{conv}}[0] \gets \underset{\bar{l}[j] + \bar{l}'[k] < \bar{l}_{\text{conv}}[1]}{\sum} f_{X}[j] \cdot f_{X'}[k]$, 
            
            \STATE $f_{\text{conv}}[1:n] \gets \left[\underset{\bar{l}[j] + \bar{l}'[k] \in [\bar{l}_{\text{conv}}[i], \bar{l}_{\text{conv}}[i+1])}{\sum} f_{X}[j] \cdot f_{X'}[k] \right]_{i \in [n]}$
    
            \STATE $f_{\text{conv}}[n+1] \gets \underset{\bar{l}[j] + \bar{l}'[k] = \bar{l}_{\text{conv}}[n+1]}{\sum} f_{X}[j] \cdot f_{X'}[k]$
        \ENDIF
        
        \RETURN $X_{\text{conv}} = (\bar{l}_{\text{conv}}, f_{\text{conv}})$
    \end{algorithmic}
\end{algorithm}

~

\newpage

Algorithm \ref{alg:range_renorm} is used to define the new grid for the convolution step. It assumes the two random variables are defined over the same number of points, $\bar{l}[0] = \bar{l}'[0], \bar{l}[-1] = \bar{l}'[-1]$, and the rest of the points are exponentially spaced with the same ratio, that is $\bar{l}[i]/\bar{l}[i-1] = \bar{l}'[j]/\bar{l}'[j-1] = r$ for some ratio $r > 1$ and all $i, j$ except for the two ends. Under these two assumptions, the sum of the two random variables is bounded in $[\bar{l}[1] + \bar{l}'[1], \bar{l}[-2] + \bar{l}'[-2]]$ and we have $\frac{\bar{l}[-2] + \bar{l}'[-2]}{\bar{l}[1] + \bar{l}'[1]} = \frac{\bar{l}[1] r^{n-1} + \bar{l}'[1] r^{n-1}}{\bar{l}[1] + \bar{l}'[1]} = r^{n-1}$, so we split the range into $n-1$ exponentially spaced bins with ratio $r$ (i.e., $n$ interior grid points).

\begin{algorithm}[H]
    \caption{Re-normalize the range for convolution \textbf{range-renorm}($X, X'$)}\label{alg:range_renorm}

    \begin{algorithmic}
       \REQUIRE$X, X'$-two random variables to be convolved
        \ENSURE$\bar{l}_{\text{conv}}$-the grid for the convolved random variable
    
        \STATE $\bar{l} \gets$ values of $X$, $\bar{l}' \gets$ values of $X'$, $n \gets (\text{size of } \bar{l}) - 2$ ~\COMMENT{Using the equal length assumption}

        \STATE $l_{\min} \gets \bar{l}[1] + \bar{l}'[1]$

        \STATE $r \gets \bar{l}[2]/\bar{l}[1]$ ~\COMMENT{Using the constant ratio assumption}

        \STATE $\bar{l}_{\text{conv}}[1:n] \gets \left[l_{\min} \cdot r^{i-1} \right]_{i \in [n]}$

        \STATE $\bar{l}_{\text{conv}}[0] \gets \bar{l}[0]$, $\bar{l}_{\text{conv}}[n+1] \gets \bar{l}[n+1]$  ~\COMMENT{Using the identical min and max values assumption}

        \RETURN $\bar{l}_{\text{conv}}$
    \end{algorithmic}
\end{algorithm}

~

The implementation also requires an algorithm (\ref{alg:disc_dist}) for discretizing a random variable to an evenly spaced grid.

~

\begin{algorithm}[H]
    \caption{Discretize a PLD realization \textbf{disc-dist}($X, \alpha, \beta$)}\label{alg:disc_dist}

    \begin{algorithmic}
       \REQUIRE$X$-PLD realization,~ $(\alpha, \beta)$ tightness parameters
        \ENSURE$U$-discretized PLD realization
    
        \STATE $(l_{\min}, l_{\max}) \gets \left(q_{X}(\beta), q_{X}(1-\beta)\right)$ 
        
        \STATE $n \gets \left\lceil\frac{l_{\max}-l_{\min}}{\alpha}\right\rceil+1$ 
        
        \STATE $\bar{l} \gets \{-\infty\} \cup \left[l_{\min} + (i-1) \cdot \alpha \right]_{i \in [n]} \cup \{\infty\}$
        
        \STATE $F_{U}[1:n] \gets \left[F_{X}(\bar{l}[i])\right]_{i \in [n]}$ ~\COMMENT{$F_{X}$ denotes the CDF of $X$ which may be a continuous or discrete distribution}
                
        \STATE $f_{U}[0] = 0$, $f_{U}[n+1] \gets 1 - F_{U}[n]$

        \STATE $f_{U}[1:n] \gets \left[F_{U}[i] - F_{U}[i-1]\right]_{i \in [n]}$

        \RETURN $U = (\bar{l}, f_{U})$
    \end{algorithmic}
\end{algorithm}

The computation of random allocation and subsampling in the remove direction requires extracting the dual from a PLD realization.

\begin{algorithm}[H]
    \caption{Extract the dual PLD \textbf{PLD-dual}($\lossRV$)}\label{alg:PLD_dual}

    \begin{algorithmic}
        \REQUIRE$\lossRV$-PLD realization
        \ENSURE$D$-PLD dual
        
        \STATE $\bar{l} \gets$ values of $\lossRV$, $n \gets (\text{size of } \bar{l}) - 2$

        \STATE $f_{D} \gets \left[e^{-\bar{l}[i]} \cdot f_{\lossRV}[i] \right]_{i \in [n]}$
        
        \STATE $f_{D}[n+1] \gets 0$, $f_{D}[0] \gets 1 - \sum_{i \in [n]} f_{D}[i]$
    
        \RETURN $D = -(\bar{l}, f_{D})$
    \end{algorithmic}
\end{algorithm}

\newpage
The next two algorithms (\ref{alg:PLD_subsam_rem}, \ref{alg:PLD_subsam_add}) compute the subsampling transformation of the PLD via Algorithm (\ref{alg:subsam_core}). Since the subsampling transformation $\phi_{\lambda}$ is non-linear, an exact lossless transformation results in a varying-width grid. Our implementation combines the transformation and re-discretization steps, subsampling directly into the re-discretized grid. We present the simpler lossless support transformation here for clarity. 

The expression $w \cdot L_{1} + (1-w) \cdot L_{2}$ should be interpreted as a new distribution supported on the union of the loss values of $L_{1}$ and $L_{2}$ and the probability vector is assigned by the corresponding distribution (if the same value appears in both distributions, its new probability is the weighted sum of its probabilities under the two distributions).

\begin{figure}[htbp]
    \begin{minipage}{0.46\textwidth}
        \begin{algorithm}[H]
            \caption{Subsample PLD (remove) \textbf{PLD-subsam-remove}($\lossRV, \lambda$)}\label{alg:PLD_subsam_rem}
        
            \begin{algorithmic}
                \REQUIRE$\lossRV$-PLD realization,~$\lambda$-sampling probability
                \ENSURE$\lossRV_{\lambda}$-subsampled PLD
        
                \STATE $\lossRV_{1} \gets \textbf{subsam-core}(\lossRV, \lambda)$
                \STATE $D \gets \text{PLD-dual}(\lossRV)$
                \STATE $\lossRV_{2} \gets \textbf{subsam-core}(-D, \lambda)$

                \RETURN $\lossRV_{\lambda} = \lambda \cdot \lossRV_{1} + (1-\lambda) \cdot \lossRV_{2}$
            \end{algorithmic}
        \end{algorithm}
        \begin{algorithm}[H]
            \caption{Subsample PLD (add) \textbf{PLD-subsam-add}($\lossRV, \lambda$)}\label{alg:PLD_subsam_add}
        
            \begin{algorithmic}
                \REQUIRE$\lossRV$-PLD realization,~$\lambda$-sampling probability
                \ENSURE$\lossRV_{\lambda}$-subsampled PLD
                
                \STATE $\lossRV_{1} \gets \textbf{subsam-core}\left(-\lossRV, \lambda\right)$

                \RETURN $\lossRV_{\lambda} = -\lossRV_{1}$
            \end{algorithmic}
        \end{algorithm}
    \end{minipage}
    \hfill
    \begin{minipage}{0.5\textwidth}
        \begin{algorithm}[H]
            \caption{Subsample core \textbf{subsam-core}($\lossRV, \lambda$)}\label{alg:subsam_core}
        
            \begin{algorithmic}
                \REQUIRE$\lossRV$-Random variable,~$\lambda$-sampling probability
                \ENSURE$\lossRV_{\lambda}$-subsampled random variable
        
                \STATE $\bar{l} \gets$ values of $\lossRV$, $n \gets (\text{size of } \bar{l}) - 2$

                \STATE $\bar{l}_{\lambda}[0] \gets \ln(1-\lambda)$
                \STATE $\bar{l}_{\lambda}\left[1:n\right] \gets \left[\ln\left(1 + \lambda (e^{\bar{l}[i]}-1) \right)\right]_{i \in \left[1:n\right]}$
                \STATE $\bar{l}_{\lambda}\left[n+1\right] \gets \bar{l}[n+1]$

                \STATE $f_{\lossRV_{\lambda}} \gets f_{\lossRV}$

                \RETURN $\lossRV_{\lambda} = (\bar{l}_{\lambda}, f_{\lossRV_{\lambda}})$
            \end{algorithmic}
        \end{algorithm}
    \end{minipage}
\end{figure}

\newpage
\section{Experimental Results}\label{apd:expRes}
In this section we provide several additional results.

Figure \ref{fig:delta_comparison} is an extended version of Figure \ref{fig:delta_comparison_partial}. It follows the setting used by \citet{CGHLKKMSZ24} to showcase their results. The number of steps is derived from the size of the training set and choice of batch size in their experimental results for Criteo Display Ads pCT (top) and Criteo Sponsored Search Conversion Log dataset (bottom) with batch sizes $1,024$ (left) and $8,192$ (right).

The Monte Carlo results were computed using importance sampling with $10^{6}$ samples and $95\%$ confidence. We note that the computation for the results derived by \citet{CGHLKKMSZ24} was performed in parallel on a cluster of 60 CPU machines.

\begin{figure}[ht]
    \centering
    \includegraphics[width=0.9\linewidth]{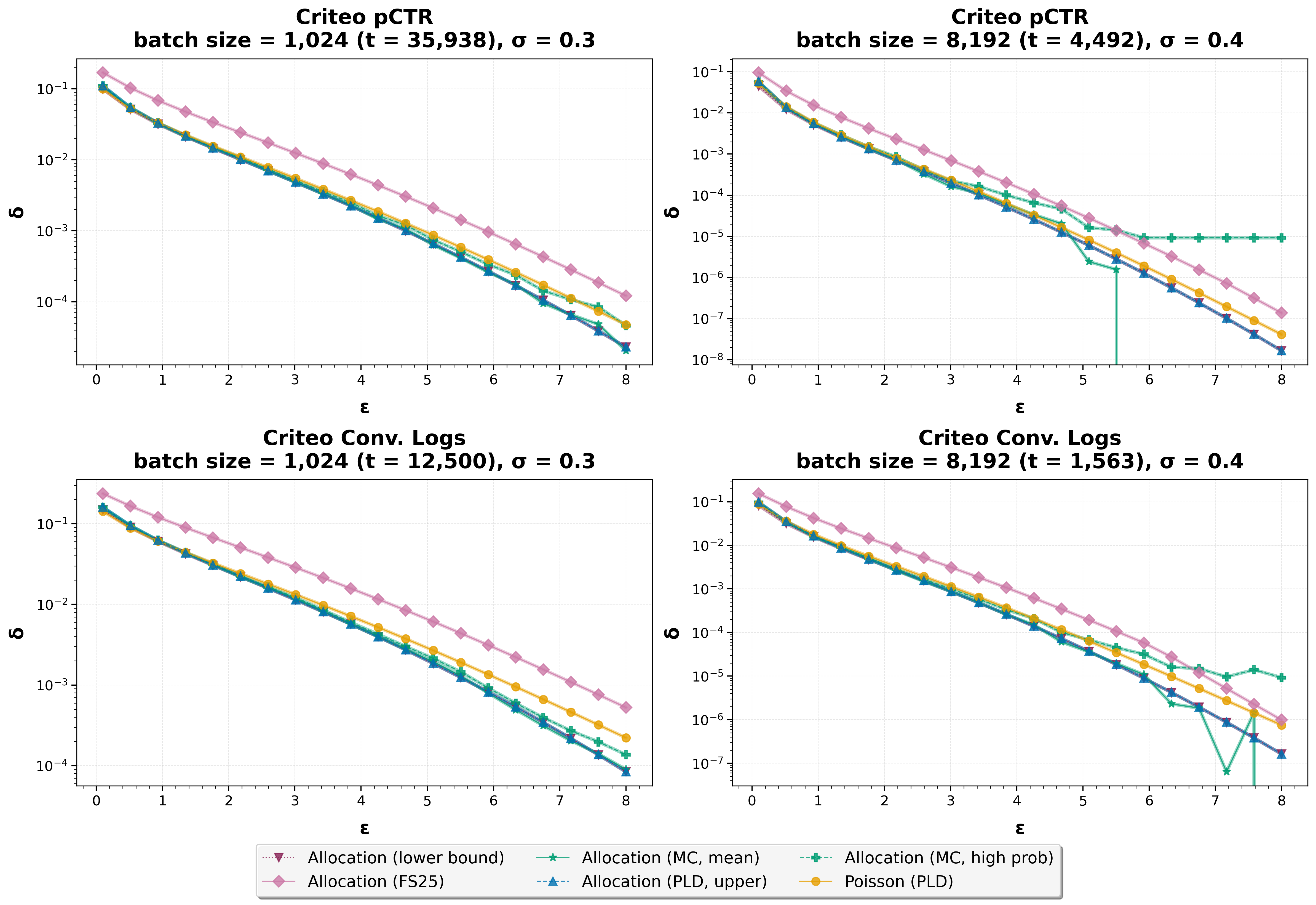}
    \caption{Comparison of the privacy profile of the Poisson scheme and various bounds for the random allocation scheme; the combined methods in \citet{FS25}, the high probability and the average estimations using Monte Carlo simulation and the lower bound by \citet{CGHLKKMSZ24}, and our numerical method, following the setting in \citet{CGHLKKMSZ24}.}
    \label{fig:delta_comparison}
\end{figure}

\newpage

\begin{figure}[htbp]
    \begin{minipage}{0.5\textwidth}
        \para{Beyond $k=1$.} The next two figures depict the effect of changing $k$, the number of allocations. Figure \ref{fig:epsilon_vs_sigma_by_k} extends the results of the main figure (Figure \ref{fig:epsilon_vs_sigma_by_t}) to the $k>1$ regime, and Figure \ref{fig:epsilon_vs_k} depicts the effect of increasing $k$ on $\varepsilon$. We note that the ``zigzag'' present in the FS25 line in the latter figure results from cases where the remainder $t \mod k$ is comparable to the partition size $t/k$. Our PLD results rely on our more careful reduction in Lemma \ref{lem:multAlloc}.
        
        The superiority of the random allocation privacy guarantees over the Poisson scheme holds in nearly all parameter regimes, as depicted in the next two plots. Figure \ref{fig:param_scan_eps} depicts the privacy parameter $\varepsilon$ for various $t$, $\sigma$, and $\delta$, after a single random allocation run or $100$ compositions, and Figure \ref{fig:param_scan_gap} depicts the gap between the upper bounds computed for random allocation and Poisson sampling. We note that there is no case where the upper bound on the privacy parameter $\varepsilon$ of Poisson is smaller than the lower bound on random allocation. The gray cells correspond to settings where the gap between the methods is smaller than the gap between the bounds. These cells are more prevalent for the $100$ epoch plots as the gap between our bounds is higher in this setting due to computational constraints.
        
    \end{minipage}
    \hfill
    \begin{minipage}{0.48\textwidth}
        \centering
        \includegraphics[width=1\linewidth]{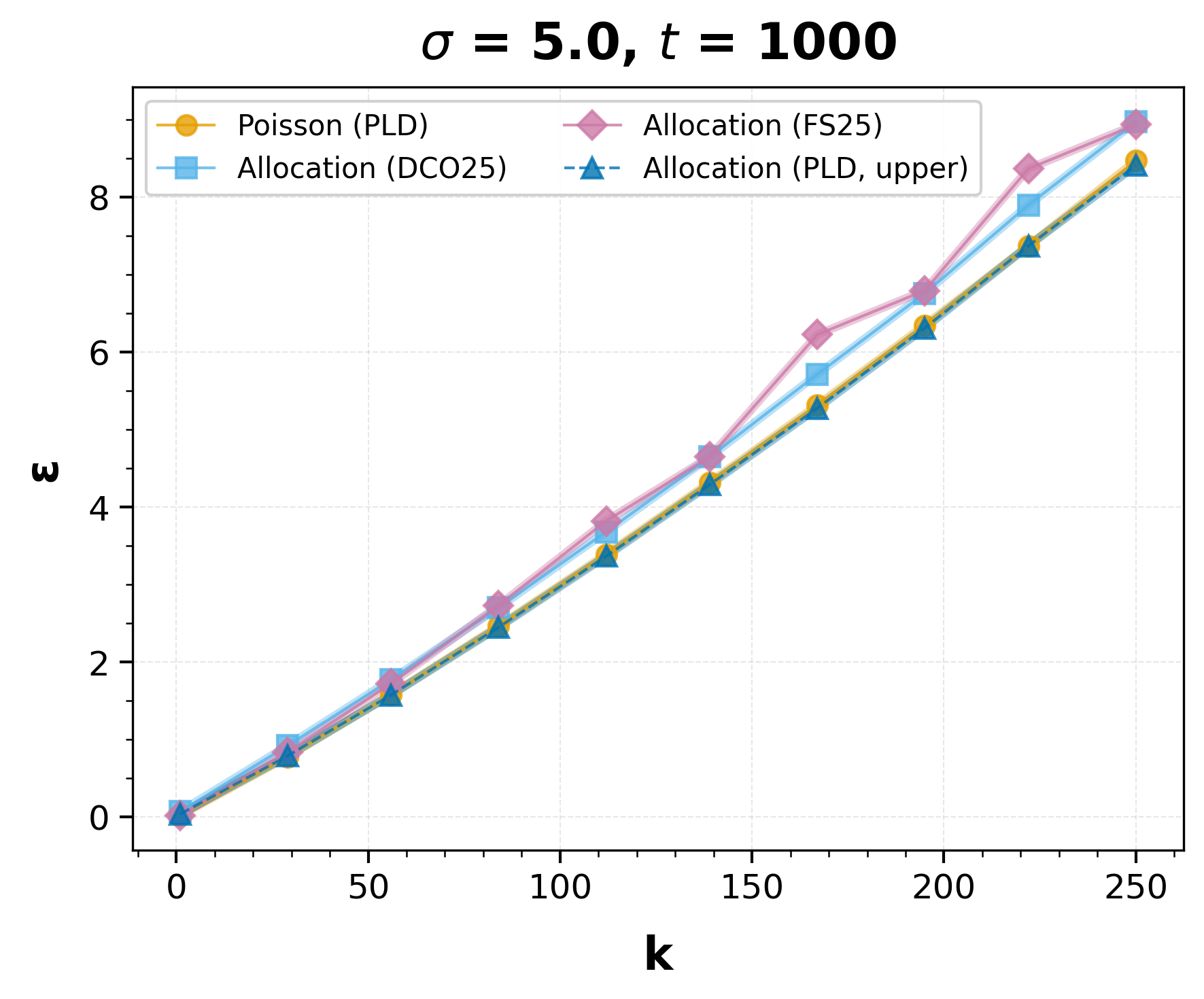}
        \caption{Upper and lower bounds on privacy parameter $\eps$ as a function of the number of allocations $k$. We compare our upper bound to upper bounds on random allocation \citep{FS25, DCO25}, and to the Poisson scheme with $\lambda = 1/t$.}
        \label{fig:epsilon_vs_k}
    \end{minipage}
\end{figure}

\begin{figure}[H]
    \centering
    \begin{subfigure}[t]{0.48\textwidth}
        \centering
        \includegraphics[width=\linewidth]{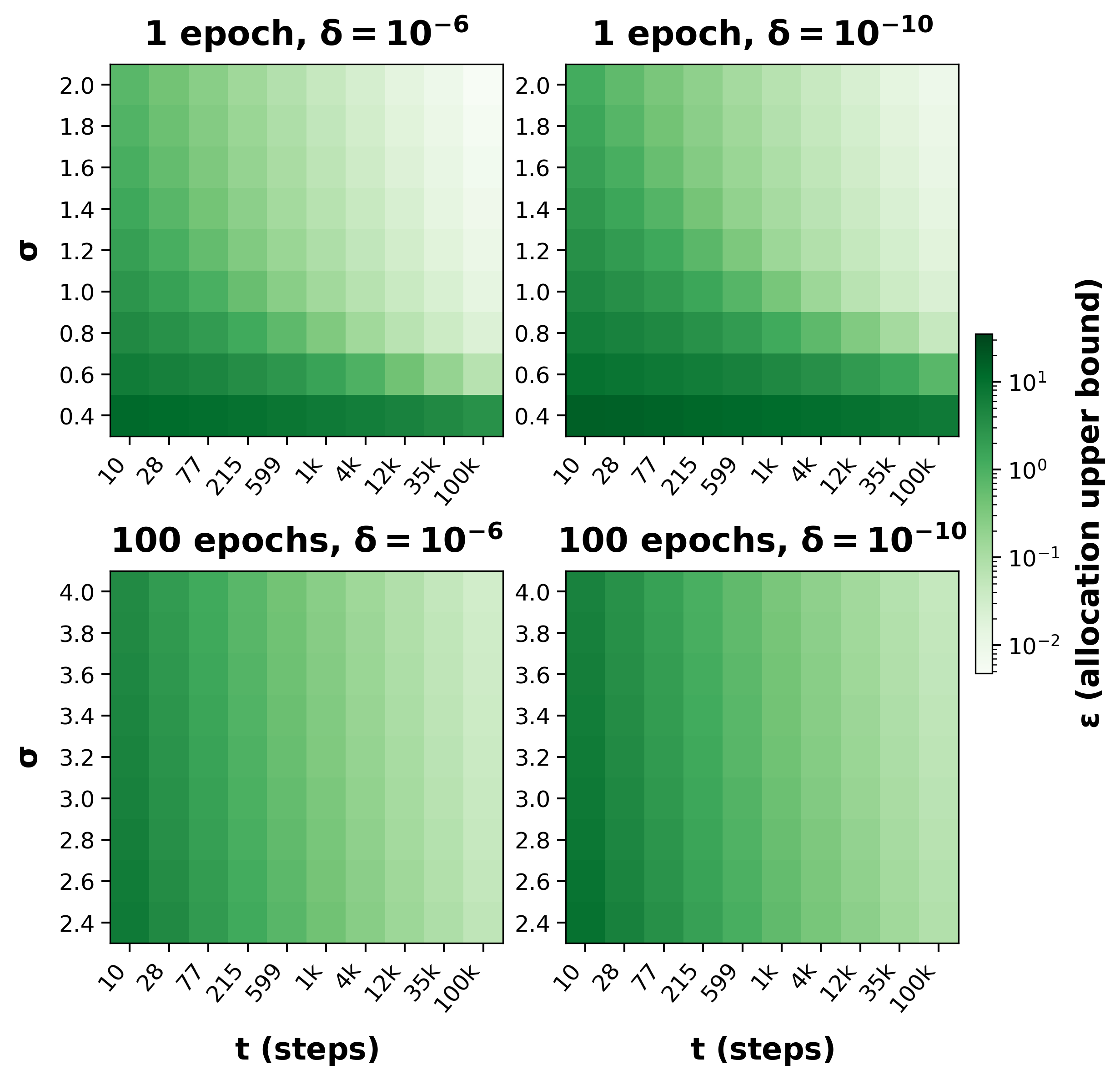}
        \caption{Upper bound on the privacy parameter $\varepsilon$ values of random allocation.
        ~
        ~}
        \label{fig:param_scan_eps}
    \end{subfigure}
    \hfill
    \begin{subfigure}[t]{0.48\textwidth}
        \centering
        \includegraphics[width=\linewidth]{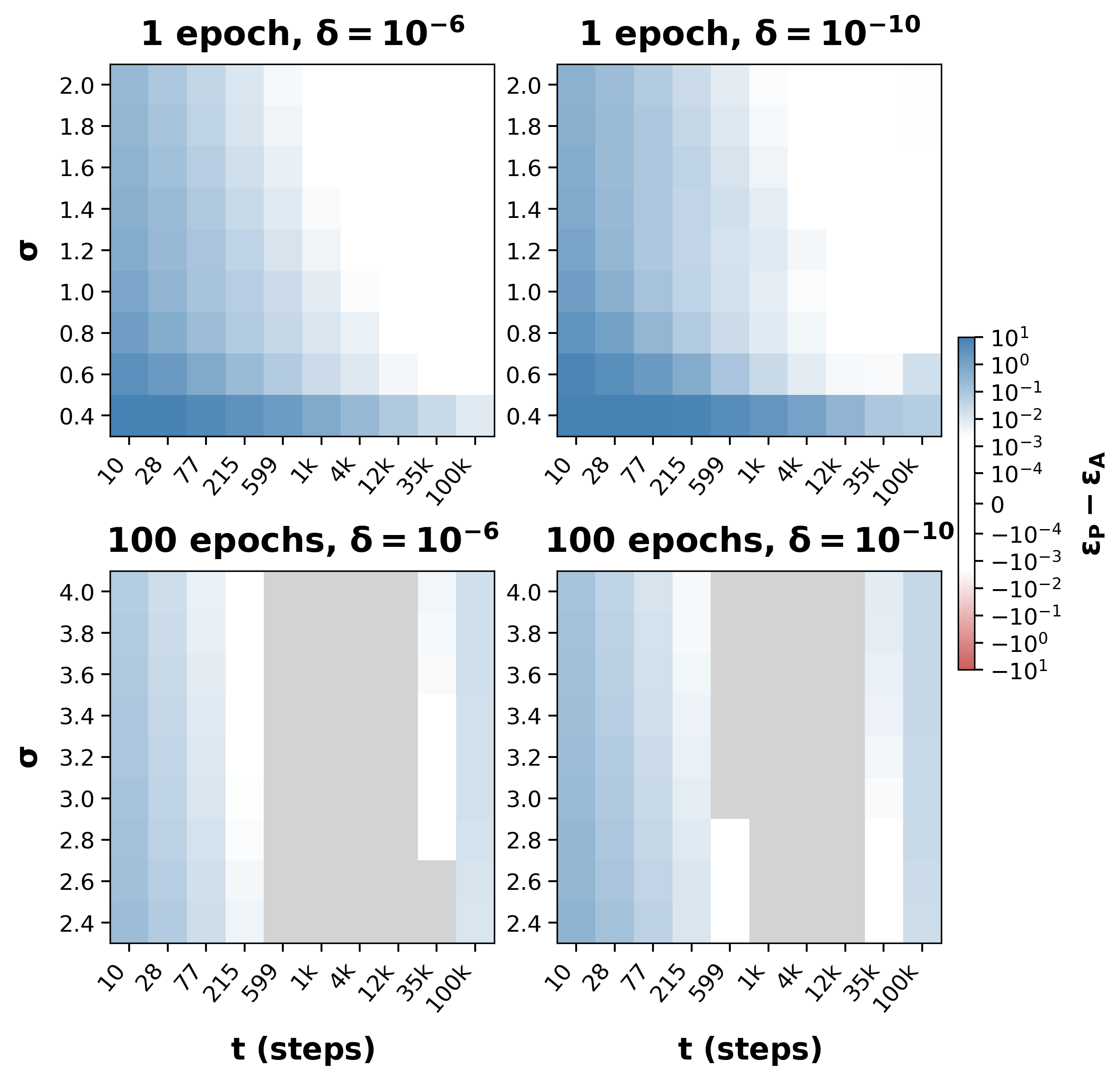}
        \caption{Gap between the upper bounds on the privacy parameter $\varepsilon$ induced by random-allocation and Poisson schemes. Gray represents values where upper bound on $\varepsilon$ for Poisson is between the upper and lower bounds for allocation.}
        \label{fig:param_scan_gap}
    \end{subfigure}
    
    \caption{Privacy parameter $\varepsilon$ for various noise scales $\sigma$, number of allocation steps $t$, values of $\delta$, and repeated composition (epochs).}
    \label{fig:param_scan}
\end{figure}

\begin{figure}[H]
    \centering
    \includegraphics[width=0.85\linewidth]{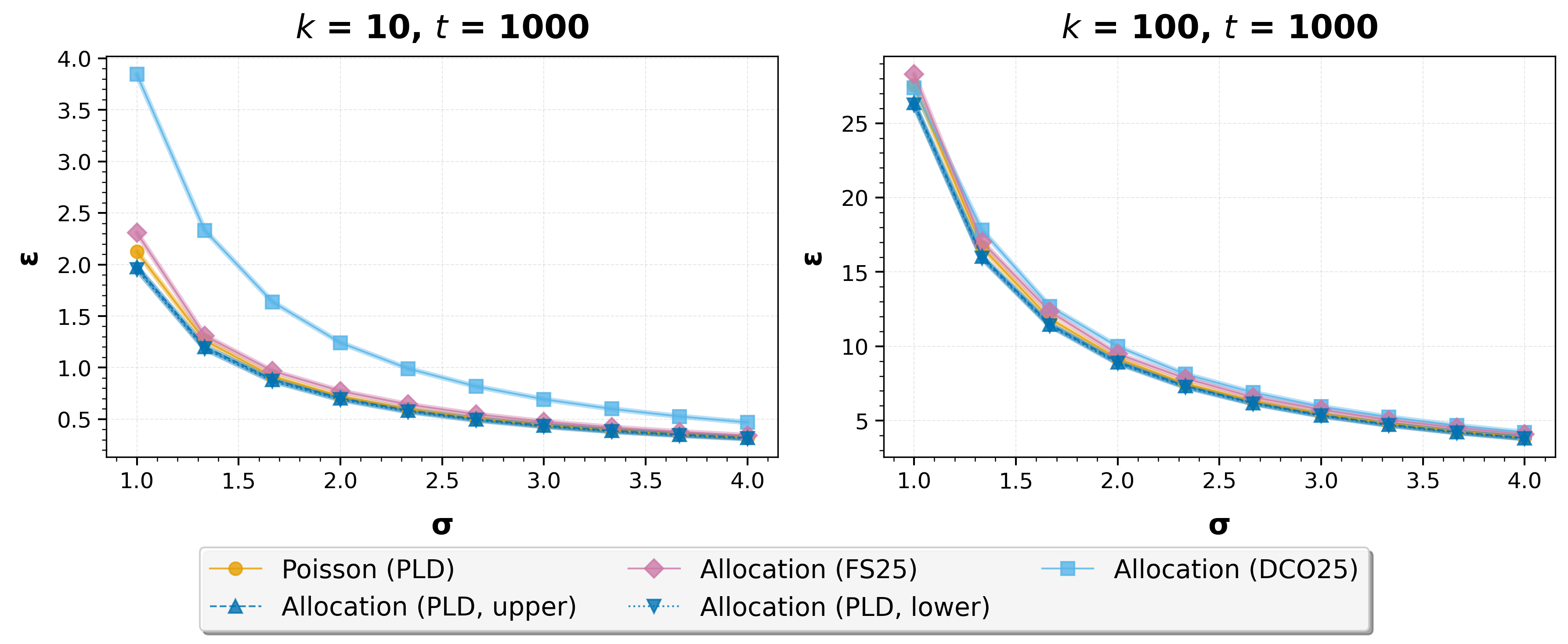}
    \caption{Upper and lower bounds on privacy parameter $\eps$ as a function of the noise parameter $\sigma$. We compare our upper and lower bounds (which are nearly identical) to upper bounds on random allocation \citep{FS25, DCO25}, and to the Poisson scheme with $\lambda = 1/t$.}
    \label{fig:epsilon_vs_sigma_by_k}
\end{figure}

\newpage

\begin{figure}[htbp]
    \begin{minipage}{0.5\textwidth}
        \para{No domination between Poisson and allocation.}
        While most numerical examples in this work show superior privacy guarantees for random allocation relative to Poisson sampling, the Poisson scheme is not dominated by random allocation, nor does it dominate it for other parameter regimes. This was first proven theoretically by \citet{CGHLKKMSZ24} for the limit of $\eps \to 0$ and $\eps \to \infty$. Figure \ref{fig:no-domination} provides a clear demonstration of this fact.
        
        We note that this near-total domination is unique to the Gaussian mechanism. Repeating the experiment with the Laplace mechanism results in more complex dynamics, as shown in Figures \ref{fig:Gauss_Laplace_by_sigma} and \ref{fig:Gauss_Laplace_by_t}. In the low privacy regime (small $\sigma$ and $t$), random allocation is significantly more private than Poisson sampling. For smaller $\varepsilon$, the roles flip and Poisson provides better privacy guarantees. As is the case with the Gaussian mechanism, the two become nearly identical as $\sigma$ and $t$ become sufficiently large.
        
    \end{minipage}
    \hfill
    \begin{minipage}{0.45\textwidth}
        \centering
        \includegraphics[width=\linewidth]{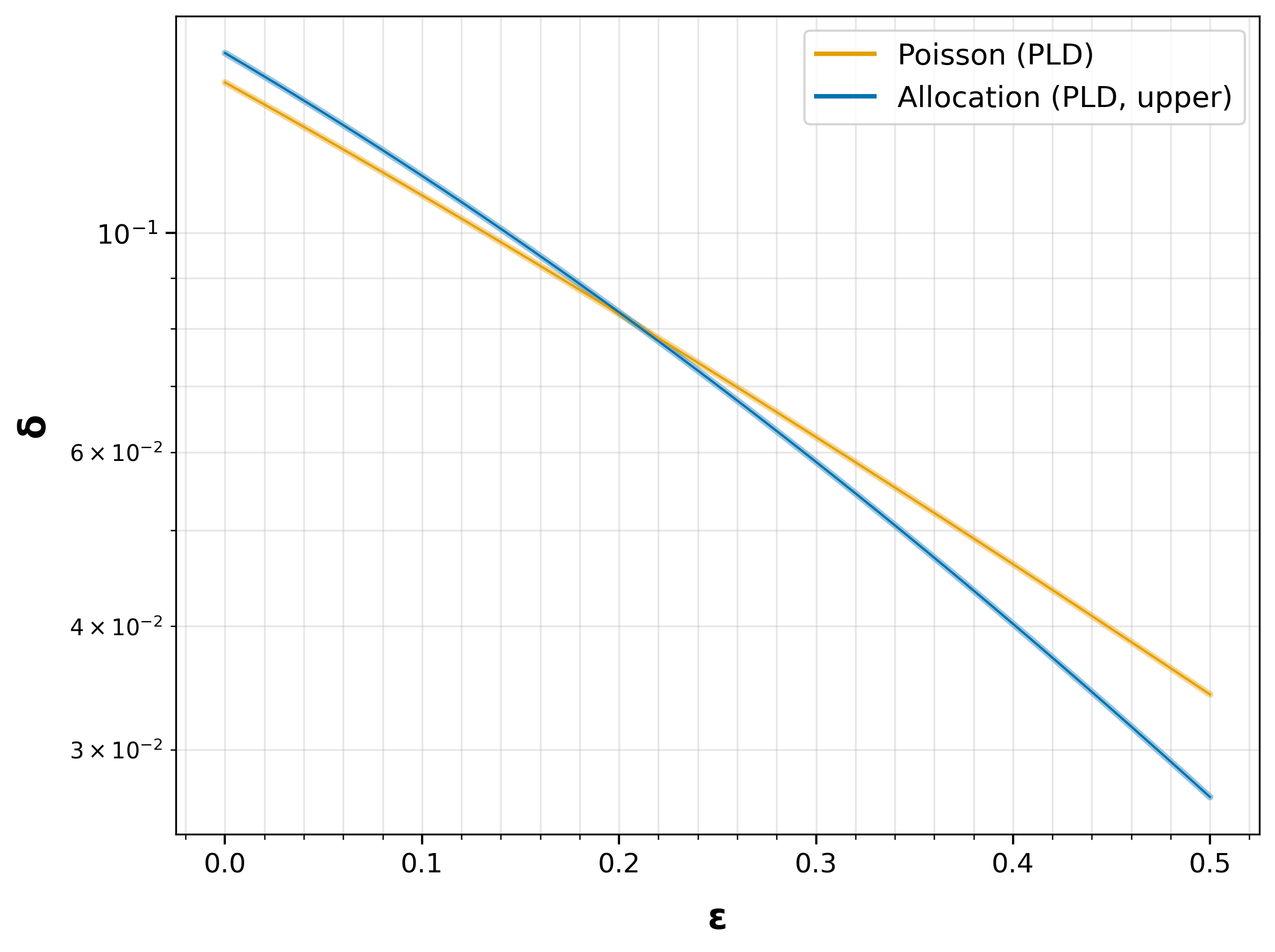}
        \caption{Privacy profile of the Poisson and random allocation schemes, clearly demonstrating they do not dominate each other.}
        \label{fig:no-domination}
    \end{minipage}
\end{figure}

\begin{figure}[H]
    \centering
    \includegraphics[width=1\linewidth]{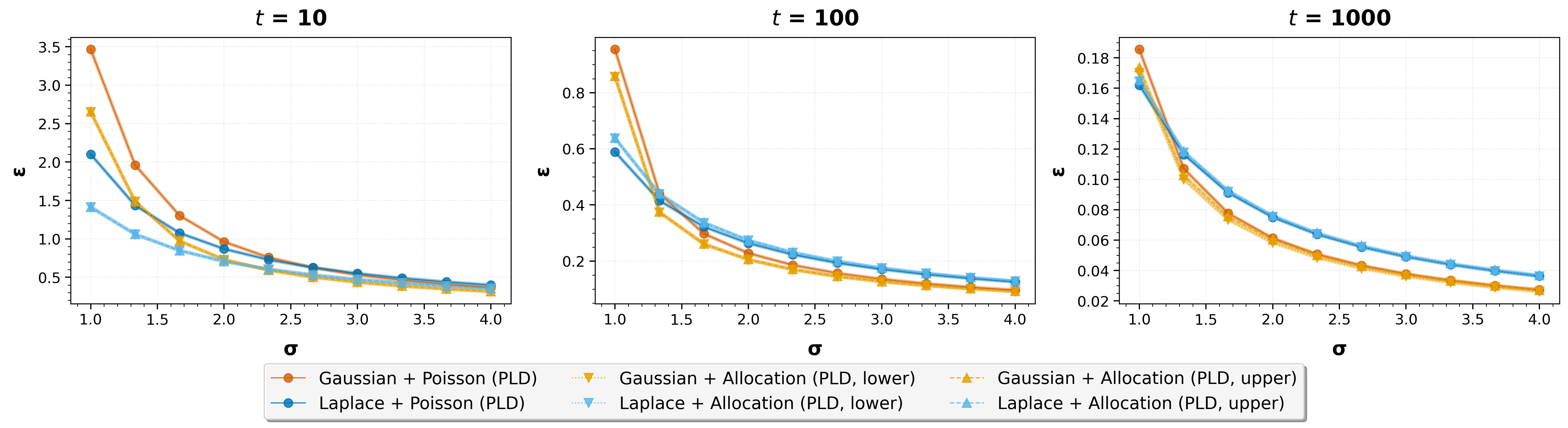}
    \caption{Comparison of the privacy parameter $\varepsilon$ for $\delta = 10^{-6}$ induced by random allocation (nearly identical upper and lower bounds) and Poisson scheme with $\lambda = 1/t$ (upper bound) as a function of the noise parameter $\sigma$ for the Gaussian and Laplace mechanisms.}
    \label{fig:Gauss_Laplace_by_sigma}
\end{figure}

\begin{figure}[H]
    \centering
    \includegraphics[width=0.65\linewidth]{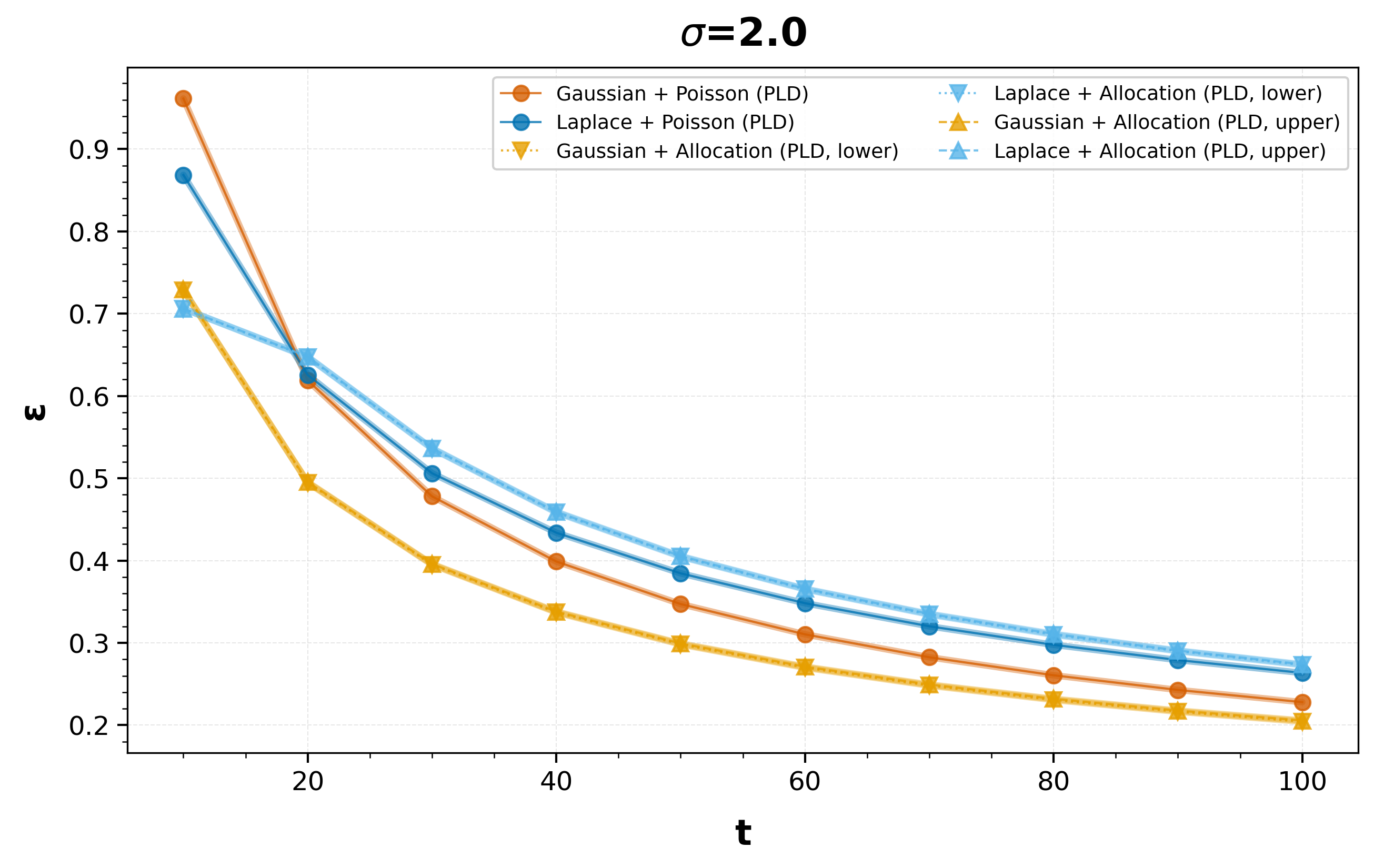}
    \caption{Comparison of the privacy parameter $\varepsilon$ for $\delta = 10^{-6}$ induced by random allocation (nearly identical upper and lower bounds) and Poisson scheme with $\lambda = 1/t$ (upper bound) as a function of the number of steps $t$ for the Gaussian and Laplace mechanisms.}
    \label{fig:Gauss_Laplace_by_t}
\end{figure}

\newpage

\showstored{mov:FFT}
\Arxiv{\para{FFT convolution.} As discussed in Section \ref{sec:numRes}, }FFT based convolution requires evenly spacing the bins in the $e^{\loss}$ space, which results in smaller bins (and thus\textemdash{}tighter bound) for $\lossRVfunc{P}{Q} \gg 1$ and larger bins (and thus\textemdash{}looser bound) for $\lossRVfunc{P}{Q} \ll -1$. Since $\rem{\delta}(\eps) \approx \prob{}{\lossRVfunc{P}{Q} > \eps}$ approximately corresponds to the right tail in the remove direction, while $\add{\delta}(\eps)$ approximately corresponds to the left tail in the add direction, this implies that FFT's results are tighter in the remove direction for a given number of bins.

Figure \ref{fig:FFT_geom_comb} portrays this phenomenon. Since the FFT-based convolution typically requires more memory while the direct multiplicative-based method results in longer runtime, there is no natural comparison between the two. We chose to use the number of bins as the comparison value, with the range of FFT bin counts at the top of each subplot and the number of multiplicative bins at the bottom. The privacy parameter $\eps$ is presented separately for the add, remove, and combined directions. In this range, the multiplicative method is tighter than the FFT in the add direction for the entire range, and only for a large number of bins in the remove direction, but the true profile of the remove direction is typically larger than add. As such, the combined method is tighter than either one in the actual privacy parameter (both).

\begin{figure}[H]
    \centering
    \includegraphics[width=1\linewidth]{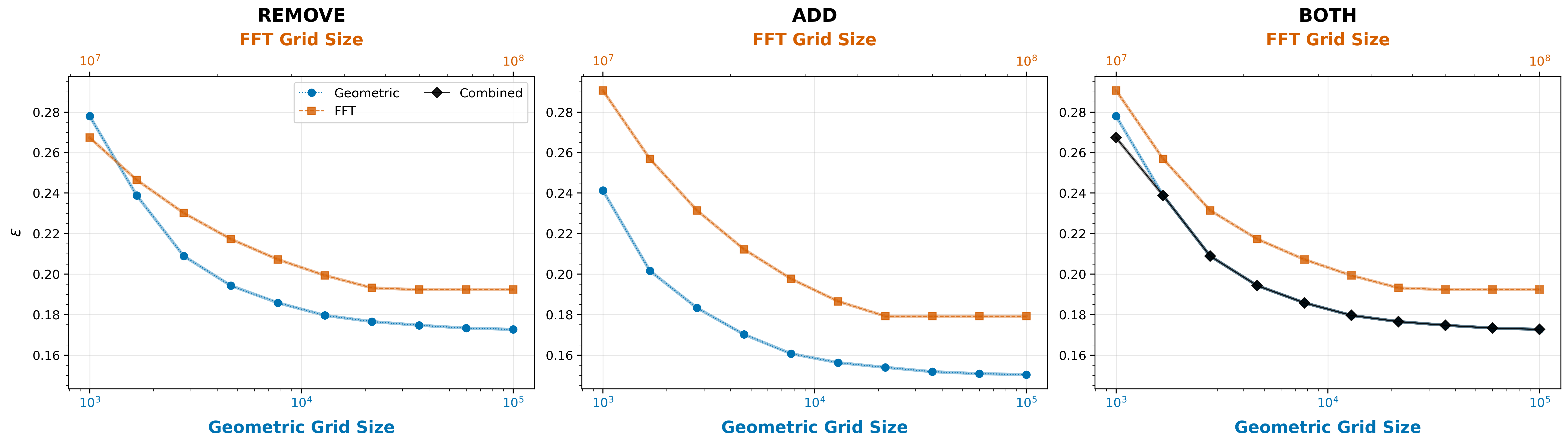}
    \caption{Comparison of $\eps$ for $\delta = 10^{-10}$ using FFT and direct multiplicative-based PLD accounting techniques for the random allocation scheme with $\sigma = 3, t = 10^{3}$, presented separately for the add, remove, and both (max of the two) directions. The x-axis represents the number of bins (bottom - multiplicative, top - FFT), both with tightness parameter $\beta = 10^{-12}$.}
    \label{fig:FFT_geom_comb}
\end{figure}

\newpage
\para{Discretization.} In Figure \ref{fig:runtime_experiment}, we compare the runtime to the gap between the computed upper and lower bounds, as determined by the discretization parameter $\alpha$. This comparison relies on the implicit assumption that the gap is indeed approximately equal to the discretization. This requires verification, as our implementation contains several rescalings of the input parameter, which is interpreted as the target final discretization. Figure \ref{fig:gap_vs_disc} demonstrates that this is indeed the case by comparing the two across various parameter regimes. Because the discretization $\alpha$ should bound the gap between the true value and the upper and lower bounds, the gap should be at most $2 \cdot \alpha$. The full configuration parameters and resulting bounds can be found in Tables \ref{tab:dominates-gap-k1} and \ref{tab:dominates-gap-kgt1}.

\begin{figure}[H]
    \centering
    \includegraphics[width=1\linewidth]{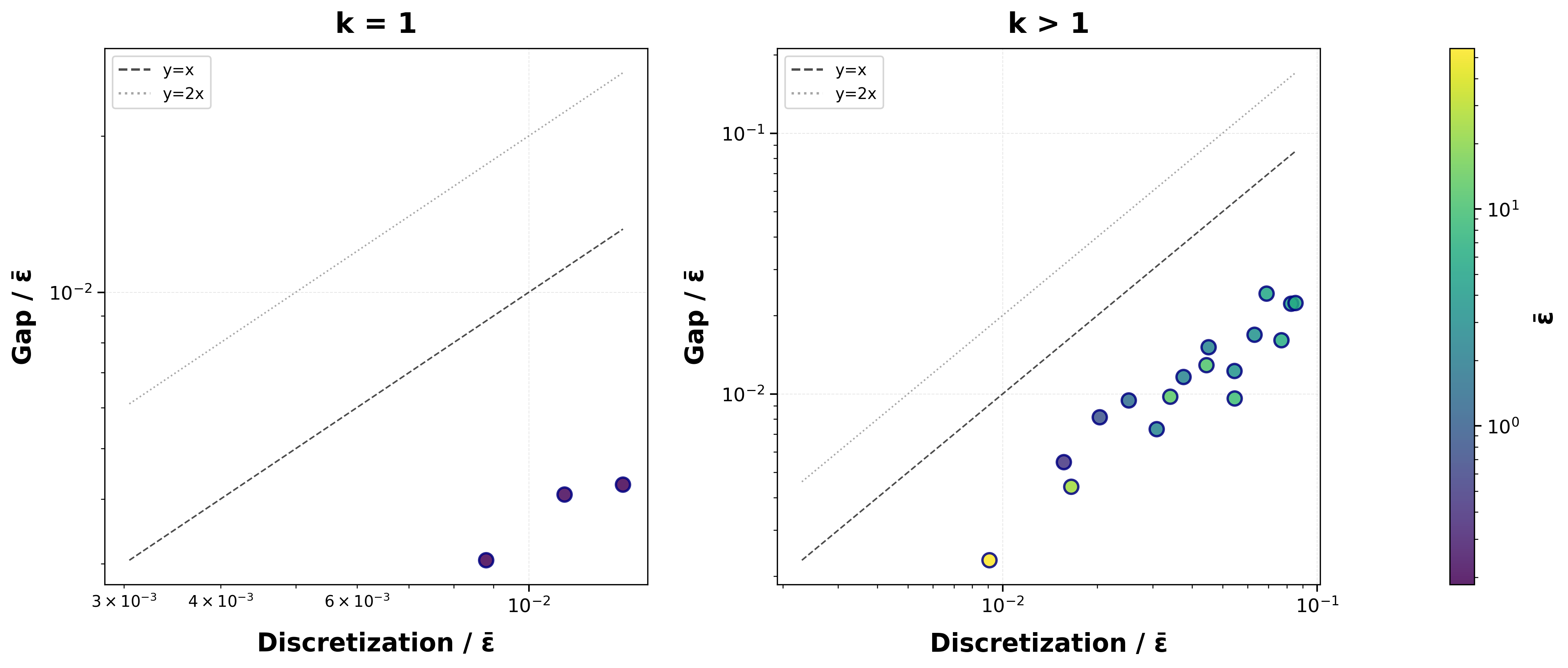}
    \caption{Gap between upper and lower bounds as a function of the discretization parameter for various parameter configurations. Values are normalized by $\varepsilon$ to align the scale of the different settings.}
    \label{fig:gap_vs_disc}
\end{figure}

\IfFileExists{./PLD_plots/dominates_gap_vs_discretization_tables}{

\begin{table}[tbp]
\centering
\small
\caption{Parameters and epsilon bounds for the $k = 1$ dominates-gap subplot.}
\label{tab:dominates-gap-k1}
\begin{tabular}{rrrrrrrrr}
\hline
Point & $\sigma$ & $t$ & $\delta$ & Discretization & Tail truncation & $\varepsilon_{\mathrm{upper}}$ & $\varepsilon_{\mathrm{lower}}$ & $\varepsilon_{\mathrm{gap}}$ \\
\hline
1 & 2.4 & 80 & $10^{-6}$ & $1.630755\times 10^{-3}$ & $10^{-8}$ & $1.85\times 10^{-1}$ & $1.85\times 10^{-1}$ & $5.64\times 10^{-4}$ \\
2 & 2 & 100 & $10^{-6}$ & $2.2798\times 10^{-3}$ & $10^{-8}$ & $2.05\times 10^{-1}$ & $2.05\times 10^{-1}$ & $8.36\times 10^{-4}$ \\
3 & 2.2 & 120 & $10^{-7}$ & $2.492507\times 10^{-3}$ & $10^{-9}$ & $1.89\times 10^{-1}$ & $1.88\times 10^{-1}$ & $8.03\times 10^{-4}$ \\
\hline
\end{tabular}
\end{table}

\begin{table}[tbp]
\centering
\small
\caption{Parameters and epsilon bounds for the $k > 1$ dominates-gap subplot.}
\label{tab:dominates-gap-kgt1}
\begin{tabular}{rrrrrrrrrr}
\hline
Point & $\sigma$ & $t$ & $k$ & $\delta$ & Discretization & Tail truncation & $\varepsilon_{\mathrm{upper}}$ & $\varepsilon_{\mathrm{lower}}$ & $\varepsilon_{\mathrm{gap}}$ \\
\hline
1 & 0.65 & 512 & 64 & $10^{-6}$ & $5\times 10^{-1}$ & $10^{-8}$ & $5.51\times 10^{1}$ & $5.5\times 10^{1}$ & $1.27\times 10^{-1}$ \\
2 & 1.8 & 120 & 2 & $10^{-6}$ & $6.97295\times 10^{-3}$ & $10^{-8}$ & $4.47\times 10^{-1}$ & $4.44\times 10^{-1}$ & $2.44\times 10^{-3}$ \\
3 & 0.7 & 384 & 32 & $10^{-6}$ & $4.07584\times 10^{-1}$ & $10^{-8}$ & $2.47\times 10^{1}$ & $2.46\times 10^{1}$ & $1.09\times 10^{-1}$ \\
4 & 1.5 & 150 & 3 & $10^{-6}$ & $1.592553\times 10^{-2}$ & $10^{-8}$ & $7.85\times 10^{-1}$ & $7.79\times 10^{-1}$ & $6.36\times 10^{-3}$ \\
5 & 1.2 & 180 & 4 & $10^{-6}$ & $3.498079\times 10^{-2}$ & $10^{-8}$ & $1.4\times 10^{0}$ & $1.38\times 10^{0}$ & $1.31\times 10^{-2}$ \\
6 & 1.1 & 190 & 6 & $10^{-6}$ & $7.49499\times 10^{-2}$ & $10^{-8}$ & $2.44\times 10^{0}$ & $2.42\times 10^{0}$ & $1.78\times 10^{-2}$ \\
7 & 0.75 & 320 & 16 & $10^{-7}$ & $4.239968\times 10^{-1}$ & $10^{-9}$ & $1.25\times 10^{1}$ & $1.24\times 10^{1}$ & $1.21\times 10^{-1}$ \\
8 & 1 & 200 & 5 & $10^{-6}$ & $9.079556\times 10^{-2}$ & $10^{-8}$ & $2.43\times 10^{0}$ & $2.4\times 10^{0}$ & $2.8\times 10^{-2}$ \\
9 & 0.75 & 320 & 16 & $10^{-6}$ & $5\times 10^{-1}$ & $10^{-8}$ & $1.13\times 10^{1}$ & $1.12\times 10^{1}$ & $1.45\times 10^{-1}$ \\
10 & 1 & 180 & 4 & $10^{-7}$ & $1.09825\times 10^{-1}$ & $10^{-9}$ & $2.45\times 10^{0}$ & $2.41\times 10^{0}$ & $3.67\times 10^{-2}$ \\
11 & 0.78 & 310 & 14 & $10^{-6}$ & $5\times 10^{-1}$ & $10^{-8}$ & $9.19\times 10^{0}$ & $9.1\times 10^{0}$ & $8.79\times 10^{-2}$ \\
12 & 0.95 & 230 & 7 & $10^{-6}$ & $1.897968\times 10^{-1}$ & $10^{-8}$ & $3.5\times 10^{0}$ & $3.46\times 10^{0}$ & $4.26\times 10^{-2}$ \\
13 & 0.9 & 250 & 6 & $10^{-6}$ & $2.064352\times 10^{-1}$ & $10^{-8}$ & $3.29\times 10^{0}$ & $3.24\times 10^{0}$ & $5.51\times 10^{-2}$ \\
14 & 0.8 & 280 & 8 & $10^{-6}$ & $3.682103\times 10^{-1}$ & $10^{-8}$ & $5.4\times 10^{0}$ & $5.27\times 10^{0}$ & $1.29\times 10^{-1}$ \\
15 & 0.82 & 305 & 11 & $10^{-6}$ & $5\times 10^{-1}$ & $10^{-8}$ & $6.55\times 10^{0}$ & $6.44\times 10^{0}$ & $1.04\times 10^{-1}$ \\
16 & 0.88 & 275 & 9 & $10^{-6}$ & $3.964506\times 10^{-1}$ & $10^{-8}$ & $4.85\times 10^{0}$ & $4.75\times 10^{0}$ & $1.06\times 10^{-1}$ \\
17 & 0.85 & 275 & 9 & $10^{-7}$ & $5\times 10^{-1}$ & $10^{-9}$ & $5.93\times 10^{0}$ & $5.8\times 10^{0}$ & $1.31\times 10^{-1}$ \\
\hline
\end{tabular}
\end{table}

}{
    \IfFileExists{./dominates_gap_vs_discretization_tables.tex}{
        
    }{
        
    }
}

\newpage
\para{General combination of sampling schemes.} The PREAMBLE setting (Fig. \ref{fig:PREAMBLE}) represents one possible way random allocation and Poisson sampling can be combined. Figure \ref{fig:double_DPSGD} compares all 4 possible combinations of two consecutive sampling schemes. This plot demonstrates the versatility of our accounting method and provides further evidence that the two sampling schemes give comparable privacy amplification. For example, Poisson $\rightarrow$ Allocation represents a $k_{2}$-out-of-$t_{2}$ random allocation scheme, where the local algorithm itself is a Poisson scheme with $t_{1}$ steps and a sampling probability of $\lambda = k_{1} / t_{1}$, using a Gaussian mechanism as its local algorithm.

\begin{figure}[H]
    \centering
    \includegraphics[width=1\linewidth]{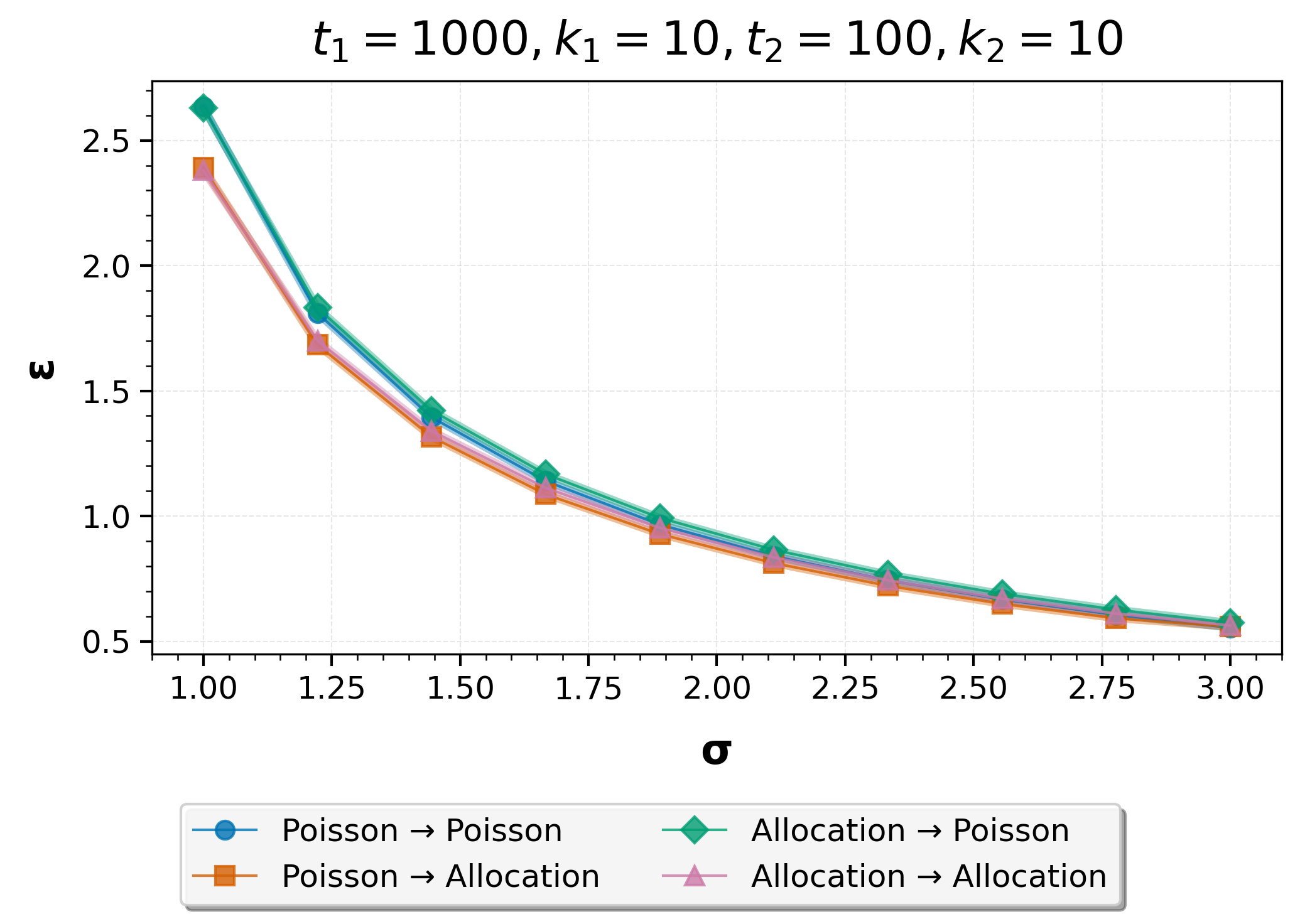}
    \caption{Comparison of the four possible combinations of two consecutive random allocation and Poisson sampling for $\delta = 10^{-8}$}
    \label{fig:double_DPSGD}
\end{figure}

\end{document}